\documentclass[10pt]{article}
\usepackage{main}
\usepackage{microtype}
\usepackage{graphicx}
\usepackage{subcaption}
\usepackage{booktabs} 
\usepackage{tabularx}
\usepackage{makecell}
\usepackage[table]{xcolor}
\usepackage{comment}
\usepackage{enumitem}
\usepackage{hyperref}
\usepackage{macros}
\usepackage{amsmath}
\usepackage{amssymb}
\usepackage{mathtools}
\usepackage{amsthm}
\usepackage{subcaption}
\usepackage{wrapfig}

\theoremstyle{plain}
\newtheorem{theorem}{Theorem}[section]
\newtheorem{proposition}[theorem]{Proposition}
\newtheorem{lemma}[theorem]{Lemma}

\theoremstyle{definition}
\newtheorem{definition}[theorem]{Definition}

\theoremstyle{remark}
\newtheorem{remark}[theorem]{Remark}

\newcommand{\AlgName}{\textsc{melinoe}}

\setopalleficon{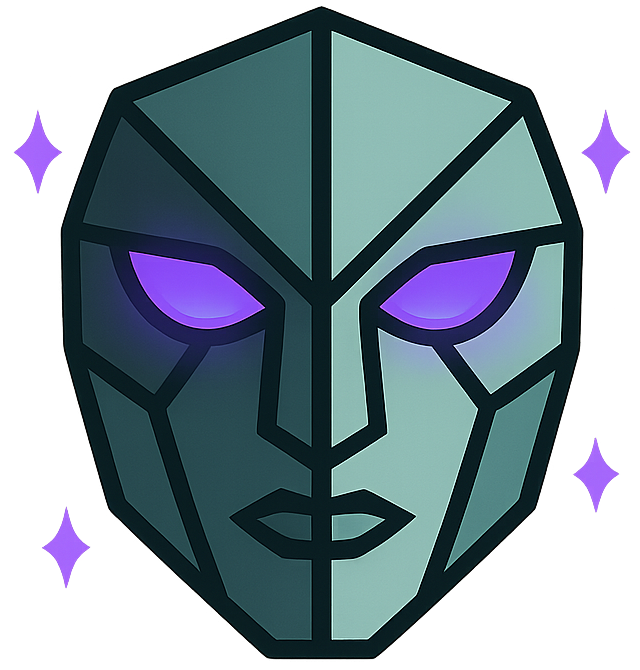}   

\title{MELINOE: Fine-Tuning Enables Memory-Efficient Inference for Mixture-of-Experts Models}
\author{Arian Raje (\texttt{araje@andrew.cmu.edu})}
\author{Anupam Nayak}
\author{Gauri Joshi}
\affil{Dept. of Electrical and Computer Engineering, Carnegie Mellon University}

\begin{document}
\maketitle

\begin{abstract}
Mixture-of-Experts (MoE) model architectures can significantly reduce the number of activated parameters per token, enabling computationally efficient training and inference. However, their large overall parameter counts and model sizes have precluded their widespread usage in resource-constrained settings as all of the parameters must still be loaded into GPU memory. Prior works aim to address this memory bottleneck by offloading certain experts into CPU memory and porting them to GPU memory only when they are activated. In practice, these methods suffer from the significant I/O latency incurred by expert transfer. We present \AlgName, a method that fine-tunes an MoE model to more strongly prefer activating a smaller number of experts per sequence. Caching these preferred experts in GPU memory reduces expert churn and CPU-GPU transfer overhead. \AlgName\ increases throughput by $1.2$–$3\times$ over efficient baselines and up to $14.7\times$ over transfer-heavy baselines while retaining or even improving the performance of the model on a downstream task, making it a reliable method for improving MoE inference efficiency.
\end{abstract}
\section{Introduction}
\label{sec:intro}

Mixture-of-Experts (MoE) models are a Large Language Model (LLM) architecture that aim to reduce the per-token computational cost of training and inference. MoEs achieve this reduction in compute by fragmenting the traditionally dense Feed Forward Network (FFN) layer in the model architecture into a sparsely gated set of multiple FFNs, referred to as ``experts'' \cite{shazeer2017, lepikhin2021gshard, 10.5555/3586589.3586709}. This reparameterization of the FFN used in transformer-based architectures makes MoE models an efficient alternative to their traditional dense counterparts. MoE models have become the flagship models for open-source model developers \cite{jiang2024mixtral} or have been offered as an alternative to dense transformers in a model suite \cite{muennighoff2025olmoe, abdin2024phi3technicalreporthighly, dai-etal-2024-deepseekmoe}. However, these gains do not come for free; while MoEs execute only a fraction of their parameters per token, their total parameter footprint remains large. During autoregressive generation, any expert may be routed to at any given step, so the corresponding expert's weights must be resident in GPU memory prior to its activation. 
\\ \\
The necessary residency of experts in GPU memory prior to activation results in one of two scenarios. First, all experts are simultaneously loaded into GPU memory, which can be prohibitive in resource-constrained settings with limited memory. This regime is also wasteful since only a small subset of experts is used per token, yet many unused expert parameters still occupy GPU memory. An alternative paradigm is where only a small subset of the total experts is kept in GPU memory while the remaining experts are $\textit{offloaded}$ into CPU memory \cite{eliseev2023fastinferencemixtureofexpertslanguage, xue2025moeinfinityefficientmoeinference, zhou2025floe}. The GPU resident experts constitute an expert $\textit{cache}$, and experts are fetched from CPU memory on demand when selected by the router. While this reduces peak GPU memory usage, it introduces a new bottleneck. Frequent expert swaps incur substantial Peripheral Component Interconnect Express (PCIe) transfer latency and can stall generation when the required expert is not already cached. The resulting churn can dominate end-to-end inference latency, especially when expert routing is diverse across tokens causing frequent transfers. When GPU memory is especially constrained, these problems are exacerbated as only a limited number of experts can be retained in the GPU-resident expert cache (Table \ref{tab:decoding_speed}). 
\\ \\ 
Prior methods propose a variety of alternative approaches to balancing memory efficiency and I/O latency. Mixtral-Offloading \cite{eliseev2023fastinferencemixtureofexpertslanguage} offloads most experts to CPU DRAM and keeps a small per-layer GPU-resident expert cache. However, I/O slowdowns still hinder improvements in end-to-end inference latency. Their aggressive mixed-precision quantization also trades model quality for memory efficiency. MoE-Infinity \cite{xue2025moeinfinityefficientmoeinference} and FLoE \cite{zhou2025floe} seek to reduce I/O transfers by using sophisticated prefetching techniques to predict which experts will be activated prior to the router selecting them. Yet, their gains depend on prediction accuracy and routing locality, where a subset of experts are activated with high probability for consecutive tokens, so diverse routing could still trigger a significant number of cache misses and evictions. Fiddler \cite{kamahori2025fiddler} reduces PCIe traffic by executing some expert computation on the CPU, but its gains are contingent on CPU capability and diminish as per-expert token counts grow, where CPU execution becomes slow and weight transfers to GPU become preferable. Meaningfully, several prior systems do not directly target reducing the $\textit{number of CPU–GPU transfers}$, which can be a dominant cost in resource-constrained MoE inference. In other words, prior work largely treats routing as fixed and optimizes around it via prefetching, offloading, or quantization. An alternative perspective is to treat routing as malleable and directly shape expert activation patterns.
\\ \\
In this vein, we present $\textbf{M}$ixture-of-$\textbf{E}$xperts with $\textbf{L}$ightweight $\textbf{I}$nference and $\textbf{N}$etwork $\textbf{O}$ffloading $\textbf{E}$fficiency, \AlgName, a framework that targets the number of CPU-GPU transfers in MoE offloading applications to reduce the amount of I/O overhead. With \AlgName, the model is fine-tuned with an auxiliary loss function that penalizes the model for inducing excessive cache transfers. This fine-tuning procedure occurs prior to deploying the MoE model in a memory-constrained system. Since individual sequences already have slight expert preferences, fine-tuning tilts the model towards using a small subset of experts more heavily while retaining or improving the performance of the model on a downstream task. \AlgName\ then employs an MLP-based predictor to forecast these preferred experts prior to decoding, making prefetching more reliable since activated experts are less diverse per sequence. \AlgName\ can improve throughput measured by tokens/s by $1.2$–$3\times$ over efficient baselines and up to $14.7\times$ over transfer-heavy baselines across model architectures and hardware configurations, demonstrating its reliability and robustness. Notably, since the fine-tuning procedure is orthogonal to prior offloading techniques, it composes naturally with other baselines and improves their effectiveness. \AlgName\ provides a practical path to MoE deployment by fine-tuning the model to reduce I/O latency without sacrificing model quality. 
\begin{table}[t]
\centering
\footnotesize
\setlength{\tabcolsep}{3.5pt}
\renewcommand{\arraystretch}{1.15}
\caption{Decoding throughput (tokens/s) on an NVIDIA H100 vs.\ cache size (fraction of experts resident in GPU VRAM). Throughput drops significantly when fewer experts are resident. Mixtral-8x7B does not fit on a single H100 GPU (80 GB VRAM) in FP16.}
\label{tab:h100_cache_coverage}

\begin{tabularx}{0.98\linewidth}{@{}l *{3}{>{\centering\arraybackslash}X}@{}}
\specialrule{0.09em}{0.2em}{0.2em}
\rowcolor{gray!12}
\textbf{Model} &
\makecell{\textbf{Cache}\\\textbf{25\% Experts}} &
\makecell{\textbf{Cache}\\\textbf{50\% Experts}} &
\makecell{\textbf{Cache}\\\textbf{All Experts}} \\
\specialrule{0.06em}{0.2em}{0.2em}
\textbf{OLMoE} & $15.68$ & $23.84$ & $37.84$ \\
\textbf{Phi-3.5-MoE} & $4.43$ & $8.04$ & $19.91$ \\
\textbf{Mixtral-8x7B} & $1.46$ & $2.32$ & -- \\
\specialrule{0.09em}{0.2em}{0.2em}
\end{tabularx}
\vspace{-4mm}
\label{tab:decoding_speed}
\end{table}

\section{Problem Setup and Motivation}
\label{sec:motivation}
MoE models replace the dense FFN in each transformer block with a sparsely gated set of FFN ``experts''. For each token in the sequence, a router selects a subset of experts to execute, producing token-level activations, while the remaining experts are inactive. We denote by $\Ex^{(\ly)}$ the set of experts in the MoE layer $\ly \in [1\cdots\Ly]$.  Let $\Exm := |\Ex^{(\ly)}|\;\forall \ly \in [1\cdots\Ly]$. Individual experts in this layer are indexed as $\Ex^{(\ly)}_i$, where $i \in \{1, \cdots \Exm\}$.
 The output of the $\ly$-th MoE layer for an input $\xin\in \mathbb{R}^d$ is given by
 \begin{equation}
\p^{(\ly)} \; = \;  \mathrm{softmax}\!\left(\mathbf{W}_r^{(\ly)} \xin\right),
\
\rv^{(\ly)} = \;  \mathrm{Top}\text{-}\K\!\left(\p^{(\ly)}\right), \nn
\end{equation}
\begin{equation}
    \y \; 
= \; \sum_{\rv^{(\ly)}_i =1} \p_{i}^{(\ly)}\,\Ex^{(\ly)}_i(\xin).\label{eq:routerout} 
\end{equation}
Here, $\mathbf{W}_r^{(\ly)} \in \mathbb{R}^{\Exm \times d}$ denotes the router weight matrix at layer $\ly$. The $\mathrm{Top}\text{-}\K$ operator selects the $\K$ experts with the highest routing probabilities, setting the corresponding entries of $\rv^{(\ly)}$ to $1$ and masking all remaining experts. Thus, $\|\rv^{(\ly)}\|_1 = \K$ . The layer output is then computed by activating only the selected experts, as shown in Equation~\eqref{eq:routerout}.
\\ \\
The hyperparameter $\K$ controls the number of experts activated per token and is model-specific, for example $\K=2$ in Mixtral-8x7B. For hidden dimension $d$ and intermediate dimension $d_{\mathrm{ff}}$, each expert is implemented as an MLP with three linear projections: a gate $\left(\W^{(\ly)}_{g,i}\in\mathbb{R}^{d_{\mathrm{ff}} \times d}\right)$, an up  $\left(\W^{(\ly)}_{u,i}\in\mathbb{R}^{d_{\mathrm{ff}} \times d}\right)$, and a down $\left(\W^{(\ly)}_{d,i}\in\mathbb{R}^{d \times d_{\mathrm{ff}} }\right)$ projection. The forward computation of expert $i$ at layer $\ly$ is
\begin{equation}
\Ex_{i}^{(\ly)}(\xin)
\;=\;
\W^{(\ly)}_{d,i}
\Big(
\phi\!\left(\W^{(\ly)}_{g,i}(\xin)\right)
\odot
\W^{(\ly)}_{u,i}(\xin)
\Big).
\end{equation}
where $\odot$ is an element-wise product and $\phi$ is the gate nonlinearity. Although only a few experts are evaluated per token, the expert parameters dominate the model's memory footprint. In OLMoE, the experts constitute $93\%$ of the weights in the model, and in Mixtral-8x7B, the experts make up $96\%$. This discordance between expert activations and their memory footprint motivates offloading applications. 

\paragraph{Expert Offloading Systems.} Prior works mitigate MoE model's memory requirements by storing most expert weights in CPU DRAM and keeping only a small subset resident in GPU VRAM as an expert cache. During decoding, when the router selects an expert that is not currently resident, the system must fetch that expert’s weights from DRAM and transfer them over the CPU–GPU interconnect, typically PCIe, into VRAM before the expert can be executed on the GPU. End-to-end decoding time therefore includes both GPU computation and transfer-induced stalls
\begin{equation}
\text{Time}_{\text{decode}} \; \approx \; \text{Time}_{\text{compute}} \; + \; N_{\text{miss}} \; \cdot \;  \text{Time}_{\text{transfer}}
\end{equation}
where $N_{\text{miss}}$ refers to the number of cache misses where the desired expert is not currently stored in GPU VRAM. Since expert weights dominate the model’s parameter count, even modest cache budgets can lead to substantial transfer overhead. In many cases, $N_{\text{miss}}$ can reach tens of thousands even over short generations (Figure \ref{fig:olmoe_transfers}). During pretraining, load balancing objectives encourage broad expert utilization, a beneficial goal in batched, multi-GPU settings to ensure experts on different devices take roughly the same amount of time to execute. However, such objectives increase the number of distinct experts touched during generation, escalating the number of transfers from the expert cache. 

\begin{figure}[t]
  \centering
  \begin{subfigure}[t]{0.49\linewidth}
    \centering
    \includegraphics[width=\linewidth]{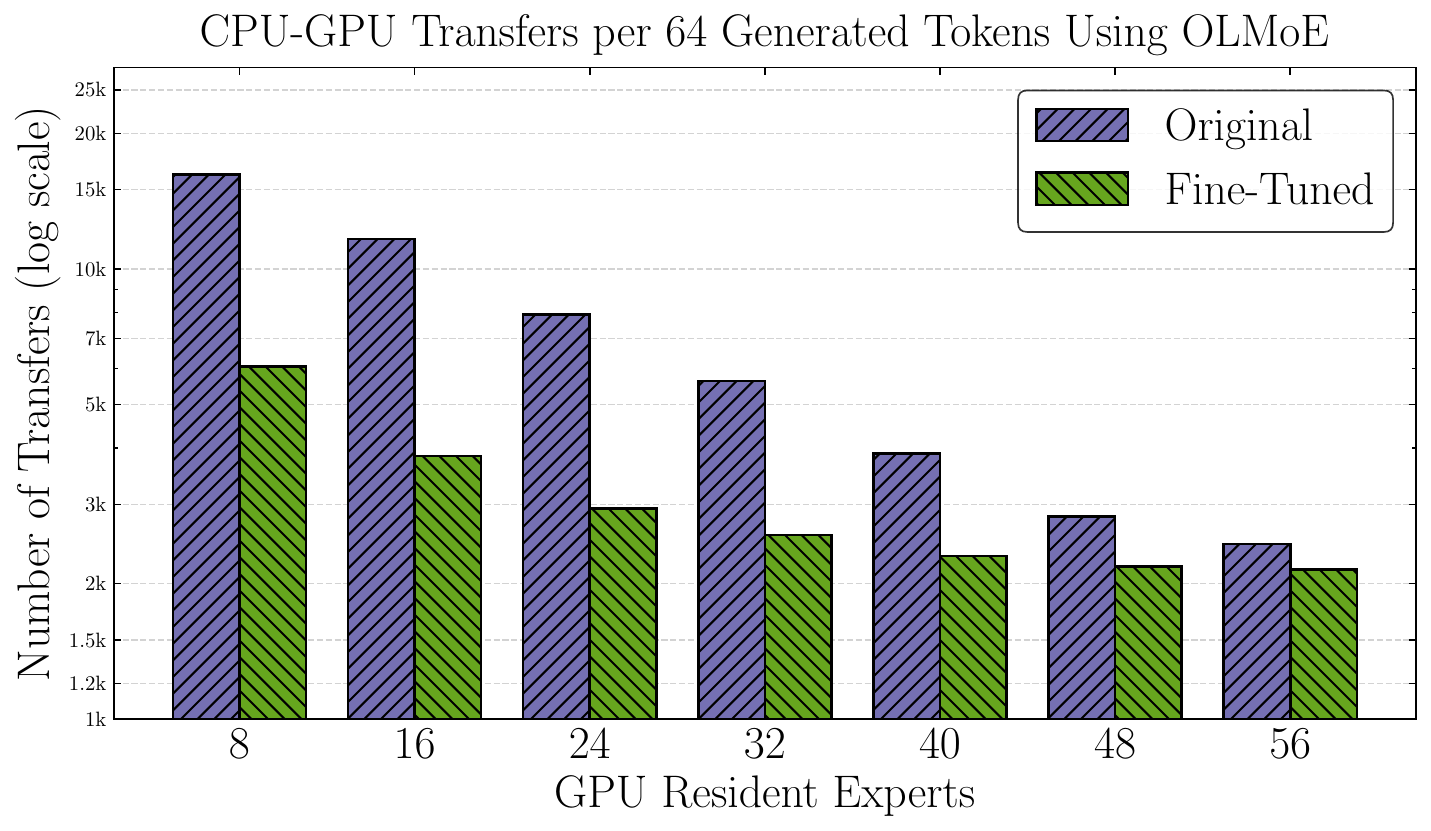}
    \caption{Number of Host-to-Device (H2D) and Device-to-Host (D2H) transfers (log scale) when generating 64 tokens using the OLMoE model (original and fine-tuned with auxiliary loss). Fine-tuning can reduce weight transfers by $3.03\times$.}
    \label{fig:olmoe_transfers}
  \end{subfigure}\hfill
  \begin{subfigure}[t]{0.49\linewidth}
    \centering
    \includegraphics[width=\linewidth]{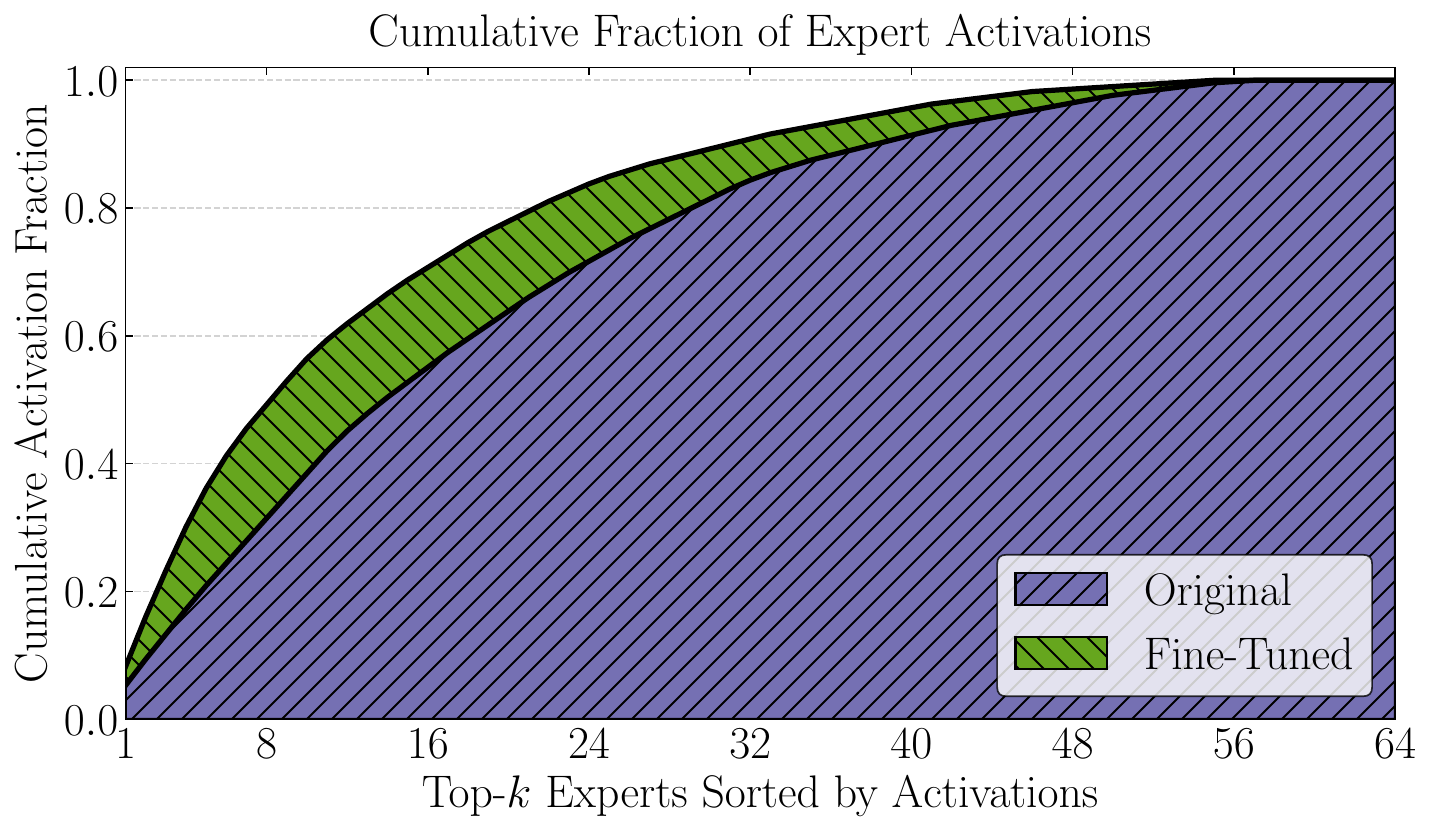}
    \caption{Experts in OLMoE exhibits only \emph{weak} concentration within a layer. Activations are distributed across many experts and the most used experts account for a small fraction of total activations. Fine-tuning increases the concentration of activations.}
    \label{fig:olmoe_concentration}
  \end{subfigure}

  \caption{OLMoE transfer behavior and routing concentration before vs.\ after fine-tuning.}
  \label{fig:olmoe_side_by_side}
  \vspace{-3mm}
\end{figure}
\vspace{0.8\baselineskip}
\paragraph{Expert Specialization.} While standard load balancing objectives spread the probability mass of expert activations across many experts, empirical evidence suggests that MoE routing is not arbitrary. Within a given sequence, the router typically exhibits a consistent but weak preference for a subset of experts, indicating a degree of sequence-level specialization \cite{liang2025modelssuitexpertoffloading, jaiswal2025findingfantasticexpertsmoes, wang2025buddymoeexploitingexpertredundancy}. This suggests that some experts contribute disproportionately to a sequence’s computation for a given prompt, even if the router still routes a nontrivial fraction of tokens elsewhere. However, in pretrained models this preference is too marginal to yield large gains from naïve caching. Even if we keep the ``most preferred'' experts in GPU VRAM, they may account for only a modest fraction of the total activations within a sequence, so cache hit rates remain limited. For example, with OLMoE we observe that, on average, the top $8$ experts by activation within a sequence account for only $\sim31\%$ of expert activations (Figure \ref{fig:olmoe_concentration}). Even though this is still greater than the expected activation rate of $12.5\%$ from all experts being activated equally, it still leaves substantial routing mass on experts that have been offloaded into CPU DRAM. Ultimately, the practical value of these preferred experts in offloading systems is paltry. 
\vspace{0.8\baselineskip}
\paragraph{Fine-Tuning to Achieve Cache-Friendliness.} Nevertheless, the observation of preferred experts motivates an important question: to what extent can we increase the concentration of expert usage within a sequence to improve expert-cache hit rates? Specifically, since the model already has a set of preferred experts per sequence, tilting the model's routing decisions towards these experts could amplify its natural per-sequence preferences. By tilting routing towards a more consistent set of experts within each sequence, we increase routing locality and make VRAM caching substantially more effective by reducing cache churn.
\\ \\
A few considerations are critical for making this approach work in practice. First, we must design an auxiliary objective that encourages the router to reuse a smaller set of experts within a sequence, increasing routing locality and reducing the number of transfers. At the same time, fine-tuning must preserve the base model’s quality. A key failure mode that would impair the model's expressiveness is router collapse where the same subset of experts become the preferred experts globally. Here, the remaining experts would not contribute to the model's predictions, degrading the fine-tuned model's integrity. The desired behavior is sequence-specific skew, where expert preferences should be persistent over the course of decoding for a given input but should vary across inputs. Finally, once such per-sequence structure is induced, we need a reliable way to predict the preferred experts before decoding begins. Accurate prediction enables proactive caching of these experts in GPU memory, making the cache substantially more effective because the same experts are likely to remain useful throughout generation. Thus, we have the following desiderata for \AlgName:
\begin{itemize}[leftmargin=*,itemsep=0pt,topsep=0pt]
    \item \textbf{Auxiliary Loss for Routing Locality}: A practical fine-tuning procedure for offloading needs an objective that penalizes too broad expert utilization per sequence.
    \item \textbf{Global Expert Usage Diversity}: Expert usage should remain diverse across sequences to preserve the model's performance and prevent expert starvation. 
    \item \textbf{Expert Activation Prediction}: Once the model has been fine-tuned and subsequently deployed on a memory-constrained device, there must be a mechanism for predicting per-sequence expert activations based on the prompt.
\end{itemize}

\begin{figure*}[t]
  \centering
  \includegraphics[width=0.94\textwidth]{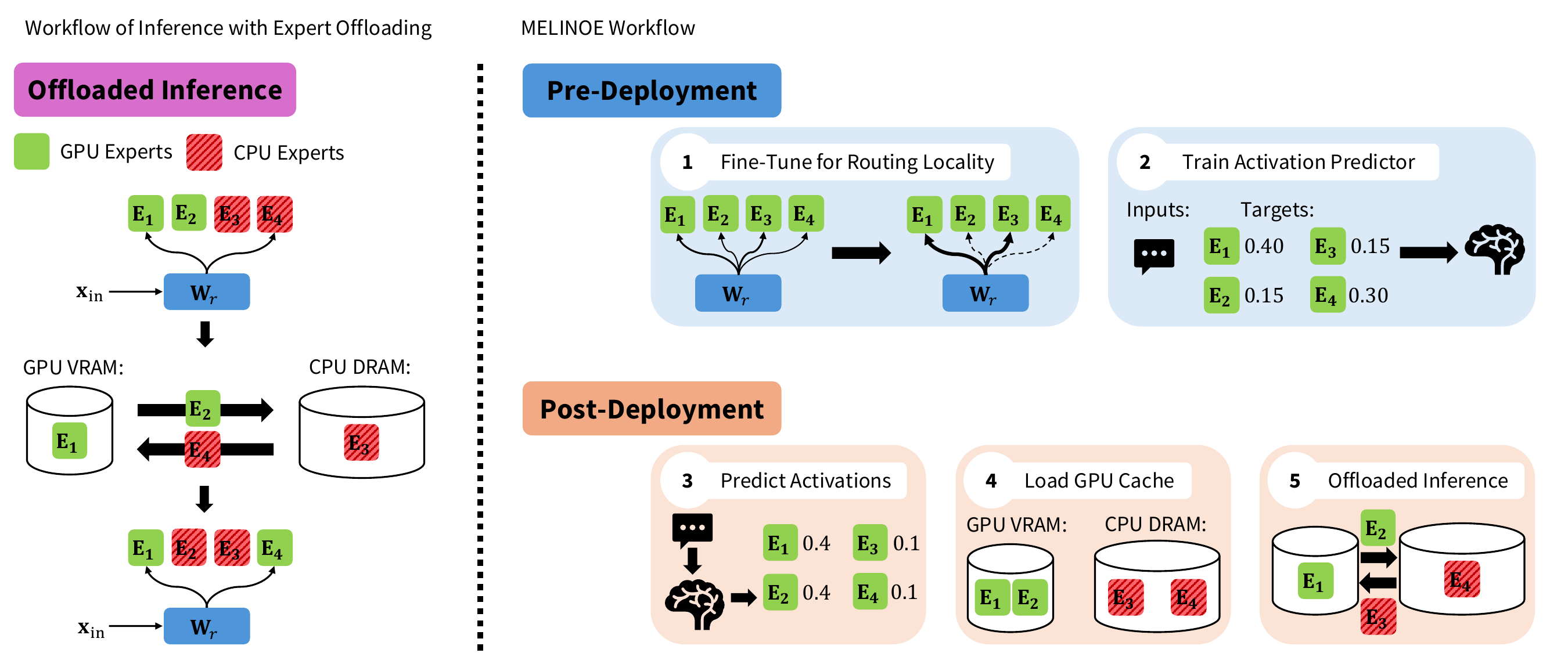}
  \caption{\textbf{Overview of \AlgName.} Pre-deployment: fine-tune the model for per-sequence routing locality and train an activation predictor. Post-deployment: predict likely experts, preload a GPU-resident cache, and run offloaded inference with fewer CPU-GPU transfers.}
  \label{fig:banner_inference_procedure}
\end{figure*}

\section{Method}
\label{sec:method}
We present \AlgName, a practical expert-offloading recipe that reduces end-to-end inference latency. We structure the method into two stages, namely pre- and post-deployment, because the resources available to a model provider differ fundamentally before and after releasing a model to a memory-constrained device. In the pre-deployment stage, the provider can access the full-precision MoE weights as well as fine-tuning data. We leverage this setting to $\textit{(i)}$ fine-tune the MoE with an auxiliary objective that increases per-sequence routing locality (Section \ref{sec:finetune}) and $\textit{(ii)}$ train an activation predictor that anticipates which experts will be most useful for a given input (Section \ref{sec:predictor}). In the post-deployment stage, the model is executed under tight VRAM budgets. Here, we use the trained predictor to proactively pre-load a GPU-resident expert cache before decoding begins for every sequence. Once these experts have been prefetched, we perform standard offloaded inference with substantially fewer transfers during generation (Section \ref{sec:post}).

\subsection{Pre-Deployment Stage}
\subsubsection{MoE Fine-Tuning Procedure}
\label{sec:finetune}
\AlgName\ hinges on the following auxiliary fine-tuning objectives that we use to reshape routing behavior. We simplify notation for the purpose of improved readability. 
\vspace{0.8\baselineskip}
\paragraph{Cache Simulation Loss $\mathcal{L}_{cs}$.} The first auxiliary objective is a cache simulation loss that directly penalizes routing patterns likely to induce expert transfers under a memory budget. We denote the router distribution for token $t \in [1, \ldots, T]$ at layer $\ell \in [1, \ldots, L]$ as $\mathbf{p}^{(\ell, t)} \in \mathbb {R}^{E}$. Since MoE inference routes each token to exactly $K$ experts, we define a binary request vector $\mathbf{r}^{(\ell, t)} \in \mathbb [0,1]^E$ by selecting the Top-$K$ entries of $\mathbf{p}^{(\ell, t)}$. We additionally define a soft cache state for token $t$ at layer $\ell$ as $\mathbf{c}^{(\ell, t)} \in \mathbb{R}_{\geq0}^E$ and initialize each $\mathbf{c}^{(\ell, 0)}$ as $\mathbf{0}$. We update this cache state using an exponentially-decayed history of past requests to approximate a recency-weighted cache. Let $\gamma \in \left[0, 1\right]$ be a decay factor and $C$ be the cache capacity. We first compute an un-normalized update according to the following rule 
\begin{equation}
\mathbf{c}^{(\ell, t+1)} \; = \; \gamma \mathbf{c}^{(\ell, t)} + \mathbf{r}^{(\ell, t)} \nn
\end{equation}
During an initial ``cache fill'' phase, we allow $\|\mathbf{c}^{(\ell, t)} \|_1$ to grow until it reaches $C$. After this point, we normalize $\mathbf{c}^{(\ell, t)}$ for subsequent tokens so that $\|\mathbf{c}^{(\ell, t)} \|_1 = C$ is preserved
\begin{equation}
\mathbf{c}^{(\ell, t+1)} \; = \; \frac{\gamma Z^{(t)}\mathbf{c}^{(\ell, t)} + \mathbf{r}^{(\ell, t)}}{Z^{(t+1)}}, \ Z^{(t+1)} \; = \; \gamma Z^{(t)} + \frac{K}{C} \nn
\end{equation}
where $Z^{(t)}$ is a scalar normalizer that maintains $\|\mathbf{c}^{(\ell, t)} \|_1 = C$ and $Z^{(1)}$ is initialized to $1$. Alternatively, one can also initialize using
a uniform vector with $\|\mathbf{c}^{(\ell, 1)} \|_1 = C$ to avoid the initial cache fill phase. In experiments, we set $\gamma=0.9$, but provide ablation studies for the impact of $\gamma$ and $C$ on throughput in Appendix \ref{subsec:soft_cache} and \ref{subsec:decay}. As $\gamma$ controls how long previous routing decisions persist in the cache, smaller values such as $\gamma = 0$ make the cache more reactive and closer to a Least Recently Used (LRU) cache. In contrast, larger values of $\gamma$ such as $\gamma = 1$ allow previous routing decisions to remain relevant over long horizons in a manner similar to a Least Frequently Used (LFU) cache.
\\ \\
The cache simulation loss can then be calculated as
\begin{equation}
\mathcal{L}_{cs} \; = \; \frac{1}{LT}\sum\limits_{\ell=1}^{L} \sum\limits_{t=1}^{T}\hspace{-2.5em}\underbrace{\sum\limits_{i=1}^{E} \mathbf{r}_{i}^{(\ell, t)}\left(1-\mathbf{c}_{i}^{(\ell, t)}\right)}_{\text{\normalsize cache miss proxy at token $t$, layer $\ell$}}
\end{equation} 
since a value of $1$ for $\mathbf{r}_{i}^{(\ell, t)}$ means the router has selected expert $\mathbf{E}_i^{(\ell)}$ while a higher value of $1-\mathbf{c}_{i}^{(\ell, t)}$ means that it is less likely that $\mathbf{E}_i^{(\ell)}$ is already in the GPU-resident cache. 
\vspace{0.8\baselineskip}
\paragraph{Rank Matching Loss $\mathcal{L}_{rm}$.} The second auxiliary objective is a rank matching loss that aligns the relative routing preferences of the fine-tuned model with the preferences of the original base model. We change the notation slightly from the previous paragraph and refer to the base model's router distribution for token $t$ at layer $\ell$ as $\mathbf{p}_b^{(\ell, t)}$ and the fine-tuned model's router distribution as $\mathbf{p}_f^{(\ell, t)}$. We count ``mistakes'' in the fine-tuned model's ordering as 
\begin{equation}
m^{(\ell, t)}=\sum\limits_{i, j \in \left[E\right]}\mathbb{I}\!\left\{\mathbf{p}_{b, i}^{(\ell, t)}>\mathbf{p}_{b, j}^{(\ell, t)}\right\}
\left[\rho-\left(\mathbf{p}_{f, i}^{(\ell, t)}-\mathbf{p}_{f, j}^{(\ell, t)}\right)\right]_+ \nn
\end{equation}
where the operation $\left[a\right]_+ := \max(0, a)$. In simple terms, if the base model is more likely to route to an expert $\mathbf{E}_i^{(\ell)}$ over $\mathbf{E}_j^{(\ell)}$, then the fine-tuned model should also prefer $\mathbf{E}_i^{(\ell)}$ over $\mathbf{E}_j^{(\ell)}$ up to some margin $\rho$. This way, the amount that the fine-tuned model prefers $\mathbf{E}_i^{(\ell)}$ over $\mathbf{E}_j^{(\ell)}$ is only nominally factored into the loss, but the underlying ordering still matters. The rank matching loss is calculated as 
\begin{equation}
\mathcal{L}_{rm} \; = \; \frac{1}{LT} \sum\limits_{\ell=1}^{L} \sum\limits_{t=1}^{T} m^{(\ell, t)}
\end{equation}
From this, we have the full loss calculation for fine-tuning
\begin{equation}
\mathcal{L} \; = \; \mathcal{L}_{nll} \; + \;  \lambda_{cs}\mathcal{L}_{cs} \; + \;  \lambda_{rm}\mathcal{L}_{rm}
\end{equation}
where $\mathcal{L}_{nll}$ is the Negative Log-Likelihood loss used in standard language modeling applications. These loss functions jointly reshape routing in a way that is aligned with memory-constrained inference. $\mathcal{L}_{cs}$ promotes within-sequence routing locality by directly penalizing excessive cache transfers, while $\mathcal{L}_{rm}$ prevents router collapse by encouraging the fine-tuned router to preserve the base model’s relative expert preferences. To further highlight the utility of these specific loss functions, we provide brief theoretical justifications for both loss functions in Appendix \ref{sec:dc}. Because we aim for fine-tuning to be memory-efficient, we update only the router parameters, specifically the router weights and gate projection, and apply low-rank adaptation (LoRA) \cite{hu2022lora} to the MLP up and down projections. 
\subsubsection{Expert Activation Predictor}
\label{sec:predictor}
In prior offloading systems, learning an expert activation predictor often provides only marginal benefit because routing is noisy, so prefetched experts are evicted before they can be reused. In \AlgName, however, fine-tuning induces stronger prompt-conditioned structure in routing, as activations are more consistent per prompt. Motivated by contextual sparsity \cite{10.5555/3618408.3619327, hou2025instructionfollowing}, we learn a prompt-conditioned predictor of expert preferences, $\Psi$. 
\\ \\ 
The activation predictor is trained on a dataset generated in the following fashion. For each prompt $\mathbf{q}$, we compute a fixed-dimensional representation using an embedding model $\Psi_{\text{EMB}}$ such that $\Psi_{\text{EMB}}(\mathbf{q}) \in \mathbb{R}^{d_{\text{EMB}}}$. To construct targets for the activation predictor, we generate a response from the MoE and record router probabilities, $\mathbf{p}^{(\ell, t)} \in \mathbb{R}^E$, for each token and layer. The supervised target for the predictor is the per-layer average router probability vector defined as 
\begin{equation}
y^{(\ell)}(\mathbf{q}) \; = \; \frac{1}{T} \sum\limits_{t=1}^{T}  \mathbf{p}^{(\ell, t)}, \ Y(\q) = \left[y^{(1)}(\q)\cdots y^{(\Ly)}(\q)\right] \nn
\end{equation}
where $Y(\mathbf{q}) \in \mathbb{R}^{L \times E}$. This yields a dataset of pairs $\{\Psi_{\text{EMB}}(\mathbf{q}_n), Y(\mathbf{q}_n)\}_{n=1}^{N}$. We train a lightweight two-layer MLP $\Psi_{\text{MLP}}: \mathbb{R}^{d_{\text{EMB}}} \rightarrow \mathbb{R}^{L \times E}$ to predict $\widehat{Y}(\mathbf{q})$ whose rows $[\widehat{Y}(\mathbf{q})]_{\ell}$ estimate the layerwise expert preference scores. We normalize the ground truth targets and train $\Psi_{\text{MLP}}$ by minimizing the KL divergence between the normalized target distribution $[Y(\mathbf{q})]_{\ell}$ and the predicted distribution by applying a row-wise softmax to $[\widehat{Y}(\mathbf{q})]_{\ell}$. Because \AlgName\ reduces routing noise by amplifying prompt-specific expert preferences, even a relatively small $\Psi_\text{EMB}$ suffices to produce a lightweight yet accurate predictor. Thus, we enable accurate prefetching, with minimal prefetching latency, and consequently reduce CPU-GPU transfers in the process. 
\subsection{Post-Deployment Stage}
\label{sec:post}
In the post-deployment stage, the fine-tuned MoE and the activation predictor are executed on a memory-constrained device. Given a prompt $\mathbf{q}$, we predict per-layer expert activations by calculating the following using the trained MLP
\begin{equation}
\Psi_{\text{MLP}}(\Psi_{\text{EMB}}(\mathbf{q})) = \widehat{Y}(\mathbf{q})
\end{equation}
For each layer $\ell \in [1, \ldots, L]$, we form a prefetch set $c^{(\ell, 1)} = \text{Top-}C([\widehat{Y}(\mathbf{q})]_{\ell})$ where $C$ is the cache capacity. We proactively load the corresponding experts into the GPU-resident cache before generation begins. To increase effective cache capacity, all expert weights are maintained in \texttt{HQQ INT4}, allowing more experts to remain resident in limited GPU memory. Experts that are not resident in GPU VRAM remain in CPU DRAM and are stored in pinned memory to accelerate transfers from CPU to GPU. Transfers are additionally non-blocking, which ensures asynchronous transfers of experts and minimizes PCIe overhead. These offloaded experts are fetched on demand when selected by the router, evicting cached experts as dictated by the cache policy. Overall, \AlgName\ localizes routing through fine-tuning, predicts the resulting per-sequence locality, and exploits it via proactive prefetching and quantized cache residency to reduce cache misses, transfers, and transfer-induced stalls during generation. In the following section, we demonstrate the effectiveness of these various design elements.

\begin{figure*}[t]
  \centering
  \includegraphics[width=\textwidth]{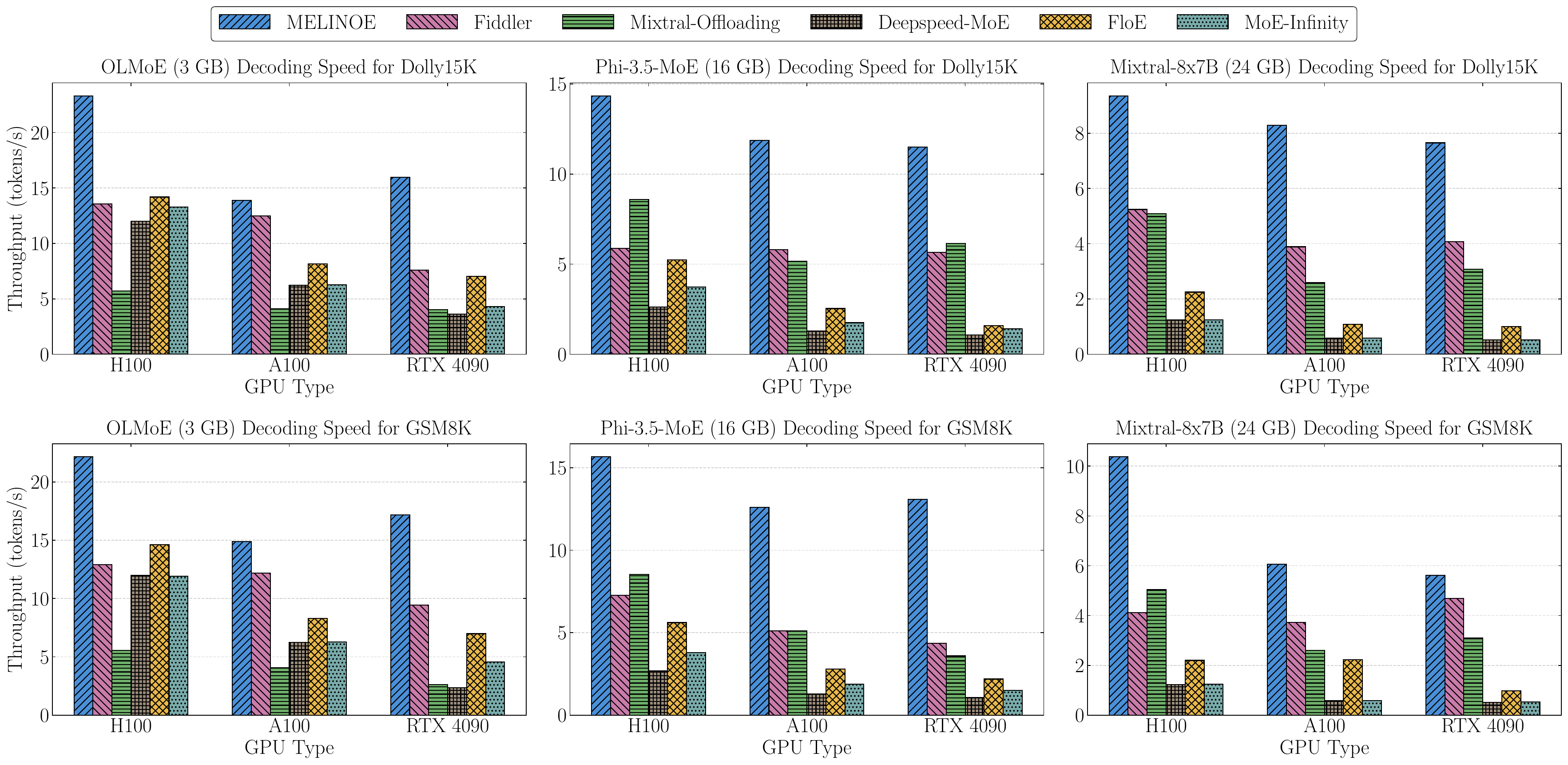}
  \caption{Throughput comparison of \AlgName\ against prior baselines across model/dataset/GPU configurations.}
  \label{fig:baseline_comparison}
\end{figure*}

\section{Results}
\label{sec:results}

\subsection{Experimental Setup}
\label{sec:setup}

\paragraph{Models and Datasets.}
We evaluate \AlgName\ on three MoE backbones that span size and granularity: \textbf{OLMoE} ($7$B params.), \textbf{Phi-3.5-MoE} ($42$B params.), and \textbf{Mixtral-8x7B} ($47$B params.). We train the activation predictor using fixed-dimensional representations from \textbf{BGE-Base-EN-v1.5} ($109$M params.) \cite{10.1145/3626772.3657878} with embedding dimension $768$. We fine-tune and evaluate across two complementary workloads. We fine-tune on \textbf{Dolly15K} \cite{DatabricksBlog2023DollyV2} (general instruction-following dataset) and \textbf{GSM8K} \cite{cobbe2021trainingverifierssolvemath} (math reasoning with longer generations) and benchmark on held-out evaluation splits.
\vspace{0.8\baselineskip}
\paragraph{Fine-Tuning Hyperparameters.} Across all models, \AlgName\ updates the router weights and gate projection layers and applies LoRA rank $r=32$ to the MLP up and down projections. Remaining weights are fixed at their pretrained initialization. For the cache-simulation loss $\mathcal{L}_{cs}$, we adopt a deliberately restrictive cache capacity to emphasize the intended deployment regime, setting the simulated cache budget to $C = \frac{E}{4}$ (e.g. $C=16$ for OLMoE). We fix the cache decay parameter to $\gamma = 0.9$ and the rank matching margin to $\rho=0.1$ throughout. Optimizer settings and fine-tuning hyperparameters are provided in Appendix \ref{subsec:opt}. 
\vspace{0.8\baselineskip}
\paragraph{Hardware and Inference Configuration.} We validate performance gains across multiple GPU types: \textbf{H100 (80GB VRAM)}, \textbf{A100 (40GB VRAM)}, and \textbf{RTX 4090 (24GB VRAM)}. To ensure that comparisons reflect realistic memory-constrained deployments, even on larger-memory accelerators, we artificially cap per-process GPU memory using PyTorch limits. We allocate 3GB for OLMoE, 16GB for Phi-3.5-MoE, and 24GB for Mixtral-8x7B. The expert cache uses an LFU eviction policy. Sensitivity to cache size and alternative budgets is deferred to Appendix \ref{subsec:cache_size}.
\subsection{Main Results}
\label{sec:main}
\begin{table}[t]
\centering
\footnotesize
\setlength{\tabcolsep}{2pt}
\renewcommand{\arraystretch}{1.12}
\caption{Downstream output quality across baselines. ROUGE-L is reported on Dolly15K. Accuracy is reported on GSM8K.}
\label{tab:quality_both_datasets}

\newcolumntype{Y}{>{\centering\arraybackslash}X}

\begin{tabularx}{\linewidth}{@{}%
>{\raggedright\arraybackslash}p{0.22\linewidth}@{}%
*{6}{Y}@{}}
\specialrule{0.09em}{0.2em}{0.2em}
\rowcolor{gray!12}
& \multicolumn{3}{c}{\textbf{Dataset: Dolly15K (ROUGE-L)}}
& \multicolumn{3}{c}{\textbf{Dataset: GSM8K (Accuracy \%)}} \\
\specialrule{0.06em}{0.15em}{0.15em}
\rowcolor{gray!6}
\textbf{Method} & \textbf{OLMoE} & \textbf{Phi-3.5-MoE} & \textbf{Mixtral-8x7B}
                & \textbf{OLMoE} & \textbf{Phi-3.5-MoE} & \textbf{Mixtral-8x7B} \\
\specialrule{0.06em}{0.2em}{0.2em}

\rowcolor{blue!8}
\textit{Base Model} &
\textit{0.1851} & \textit{0.2067} & \textit{0.2159} &
\textit{79.21} & \textit{55.45} & \textit{77.23} \\

\AlgName &
\textbf{0.2486} & \textbf{0.2270} & \textbf{0.2361} &
\textbf{80.20} & \textbf{63.37} & \textbf{79.21} \\

Fiddler &
0.1851 & 0.2067 & 0.2159 &
79.21 & 55.45 & 77.23 \\

Mixtral-Offloading &
0.1734 & 0.2025 & 0.2086 &
72.28 & 51.49 & 61.39 \\

DeepSpeed-MoE &
0.1851 & 0.2067 & 0.2159 &
79.21 & 55.45 & 77.23 \\

FLoE &
0.1775 & 0.1884 & 0.2212 &
63.34 & 53.47 & 60.40 \\

MoE-Infinity &
0.1851 & 0.2067 & 0.2159 &
79.21 & 55.45 & 77.23 \\

\specialrule{0.09em}{0.2em}{0.2em}
\end{tabularx}
\vspace{-2mm}
\end{table}
We validate the performance of \AlgName\ against five prior baselines: Fiddler \cite{kamahori2025fiddler}, Mixtral-Offloading \cite{eliseev2023fastinferencemixtureofexpertslanguage}, Deepspeed-MoE \cite{ae449111733a42c5980594f9133812c8}, FLoE \cite{zhou2025floe}, and MoE-Infinity \cite{xue2025moeinfinityefficientmoeinference}. Across every configuration, \AlgName\ significantly improves throughput relative to the baseline methods (Figure \ref{fig:baseline_comparison}). Using OLMoE, \AlgName\ achieves $22.16$-$23.32$ tokens/s on the H100 setup and $15.99$-$17.17$ tokens/s on the RTX 4090 setup. In comparison, the best competing baseline achieves only $14.19$-$14.62$ tokens/s on the H100 setup (FLoE) and $7.61$-$9.44$ tokens/s on the RTX 4090 setup (Fiddler). These improvements persist with the larger architectures. Phi-3.5-MoE reaches $14.34$-$15.67$ tokens/s on the A100 setup whereas the next closest baselines (Fiddler and Mixtral-Offloading) only attain $5.11$-$5.81$ tokens/s. Ultimately, MELINOE improves over the best competing baseline by $1.2$-$3\times$ and yields substantially larger gains against baselines that incur frequent weight transfers. For instance, \AlgName\ achieves a $14.7\times$ throughput gain over Deepspeed-MoE when using Mixtral-8x7B on the RTX 4090 setup, showing considerable improvements on resource-constrained hardware configurations. 
\\ \\ 
Crucially, \AlgName's throughput gains do not come at the expense of task performance (Table \ref{tab:quality_both_datasets}). Note that the Fiddler, DeepSpeed-MoE, and MoE-Infinity baselines do not change the base model weights and therefore have the same performance as the base model. On Dolly15K, \AlgName\ achieves the best ROUGE-L with OLMoE ($0.2486$), Phi-3.5-MoE ($0.2270$), and Mixtral-8x7B ($0.2361$). On GSM8K, \AlgName\ again attains the highest accuracy across all three architectures, with the most pronounced improvement occurring with Mixtral-8x7B where \AlgName\ outperforms the next closest baseline by $4.35\%$. Tilting routing towards more persistent per-sequence experts can maintain or even improve quality on the downstream task.
\subsection{Ablation Studies and Additional Analysis}
\label{sec:ablations}

\paragraph{Relative Impact of Fine-Tuning vs. Prefetching.}
\begin{wrapfigure}{r}{0.50\linewidth}
  \vspace{-15mm}
  \centering

  \includegraphics[width=\linewidth]{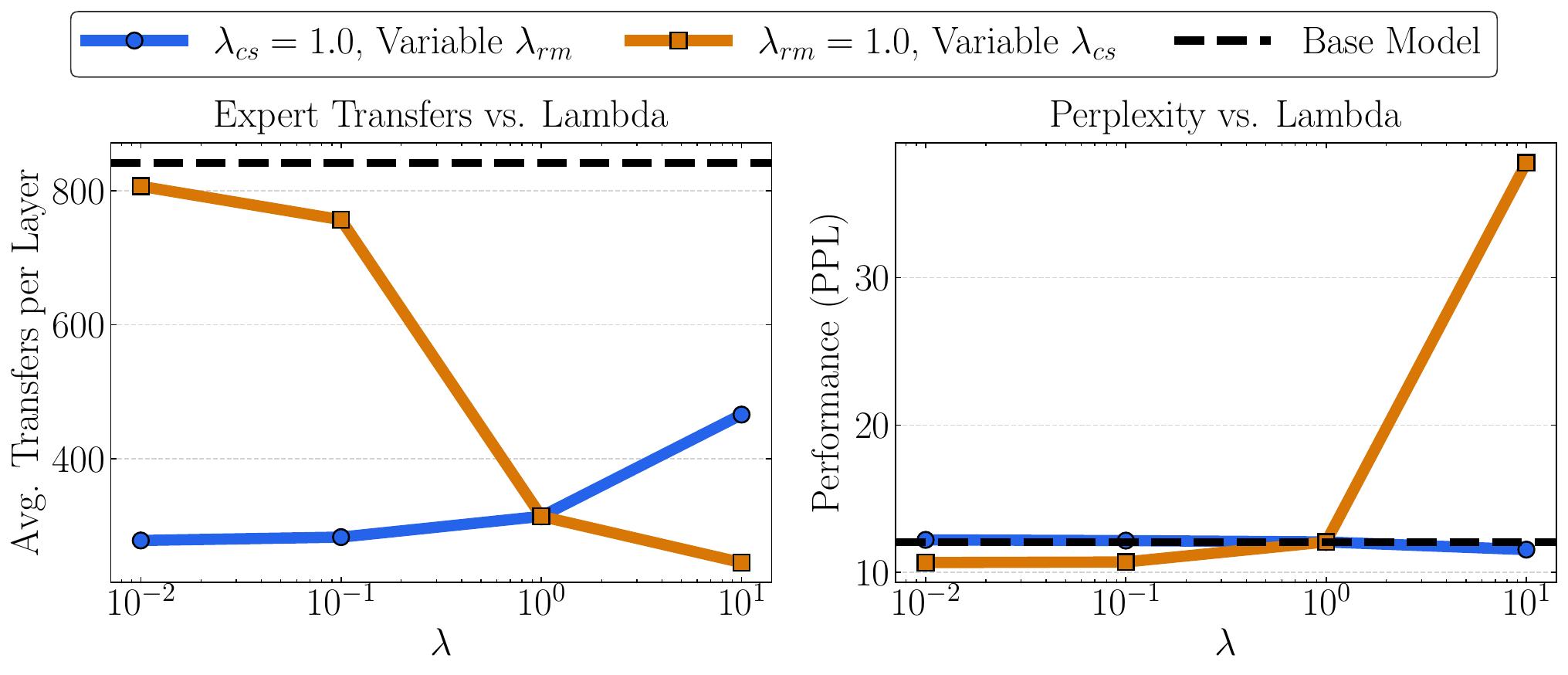}
  \captionof{figure}{Impact of varying $\lambda_{cs}$ and $\lambda_{rm}$ on number of expert transfers and model performance (OLMoE, 64 output tokens).}
  \label{fig:lambda_ablation}

  \vspace{2mm}

  \includegraphics[width=\linewidth]{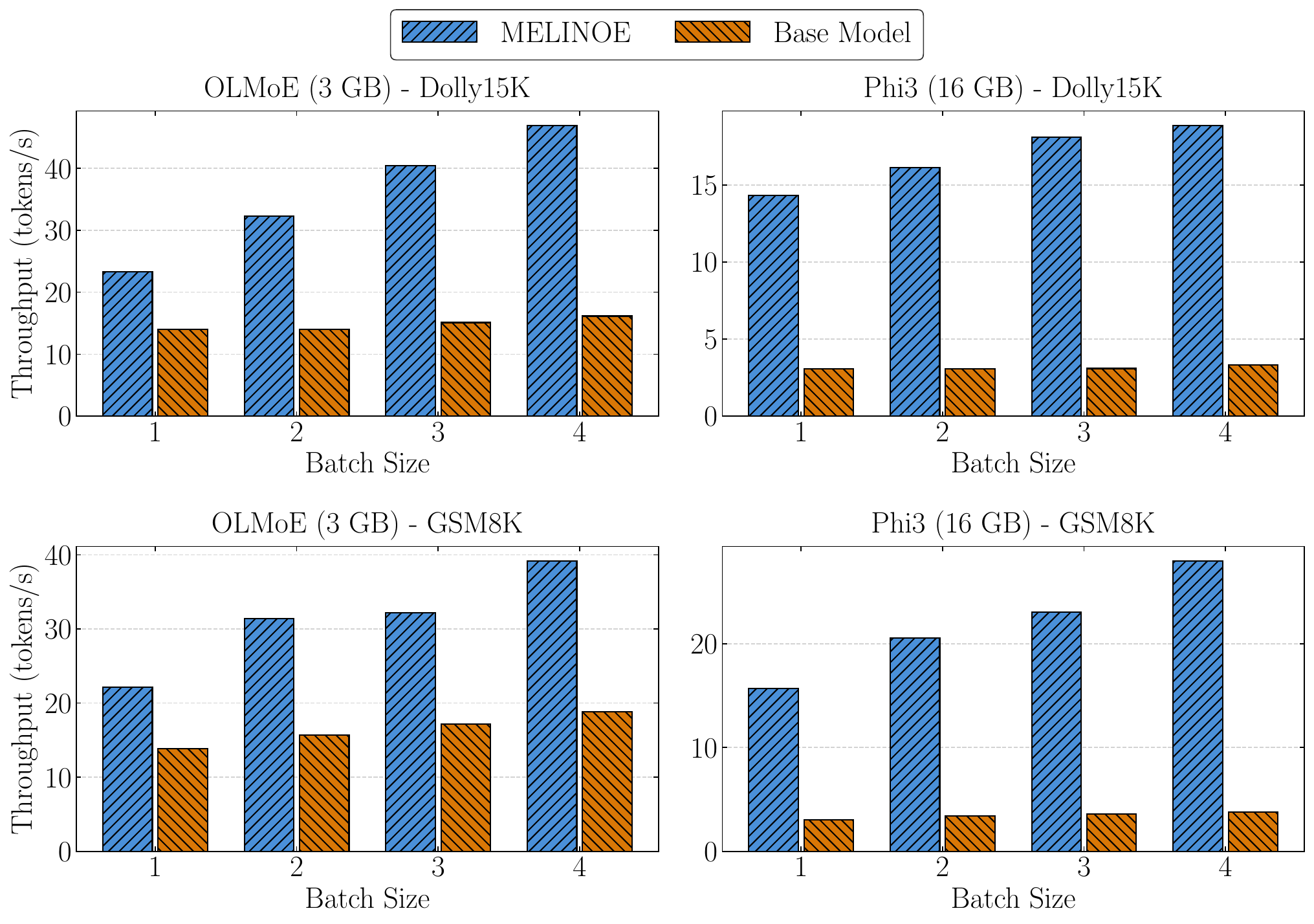}
  \captionof{figure}{Throughput of \AlgName\ at various batch sizes relative to the base model with limited GPU VRAM.}
  \label{fig:batch_size}
  \vspace{-8mm}
\end{wrapfigure}
First, we aim to disambiguate the influence of fine-tuning and prefetching on the performance of \AlgName. Table \ref{tab:prefetch_ablation_both} quantifies the impact of these disparate components for OLMoE and Mixtral-8x7B on Dolly15K and GSM8K. We conclude that the primary factor in throughput improvements is fine-tuning, which can substantially reduce the number of CPU-GPU transfers. In OLMoE, fine-tuning the model can result in $3\times$ fewer transfers relative to the base model. For Mixtral-8x7B, this reduction in transfer counts is especially valuable as transfers are slow in coarse-grained MoE architectures where experts themselves are larger. Even with PCIe $5$ x$16$, a single expert transfer for Mixtral-8x7B without quantization can take $5$-$6$ ms. Prefetching provides supplementary benefits, but the improvements are marginal relative to the gains achieved through fine-tuning alone. Nonetheless, prefetching itself takes roughly $0.05$ seconds but can still reduce end-to-end latency by up to $0.8$ seconds. 
\begin{table}[!t]
\centering
\footnotesize
\setlength{\tabcolsep}{2.5pt}
\renewcommand{\arraystretch}{1.12}
\caption{Impact of fine-tuning and prefetching (64 output tokens). Each entry reports throughput (tokens/s) with average transfers per layer in parentheses (Tx/L).}
\label{tab:prefetch_ablation_both}

\newcolumntype{Y}{>{\centering\arraybackslash}X}

\begin{tabularx}{\linewidth}{@{}>{\raggedright\arraybackslash}p{0.34\linewidth}@{} *{4}{Y}@{}}
\specialrule{0.09em}{0.2em}{0.2em}
\rowcolor{gray!12}
&
\multicolumn{2}{c}{\textbf{Dataset: Dolly15K}} &
\multicolumn{2}{c}{\textbf{Dataset: GSM8K}} \\
\specialrule{0.06em}{0.15em}{0.15em}
\rowcolor{gray!6}
\textbf{Setting} &
\makecell{\textbf{OLMoE}\\\textbf{($C=16$)}} &
\makecell{\textbf{Mixtral-8x7B}\\\textbf{($C=5$)}} &
\makecell{\textbf{OLMoE}\\\textbf{($C=16$)}} &
\makecell{\textbf{Mixtral-8x7B}\\\textbf{($C=5$)}} \\
\specialrule{0.06em}{0.2em}{0.2em}

\textbf{Base Model} &
$15.15\,(727)$ &
$3.37\,(105)$ &
$15.10\,(771)$ &
$3.59\,(100)$ \\

\textbf{Fine-Tuned Model} &
$23.35\,(240)$ &
$8.36\,(48)$ &
$24.88\,(255)$ &
$7.82\,(52)$ \\

\textbf{Fine-Tuned Model + Prefetch} &
$\mathbf{24.10\,(229)}$ &
$\mathbf{8.98\,(46)}$ &
$\mathbf{25.00\,(250)}$ &
$\mathbf{8.60\,(48)}$ \\

\specialrule{0.09em}{0.2em}{0.2em}
\end{tabularx}
\end{table}
\vspace{0.8\baselineskip}
\paragraph{Contribution of Loss Functions on Performance.}
Our fine-tuning objective uses two weighting coefficients, $\lambda_{cs}$ and $\lambda_{rm}$, to control the relative influence of the auxiliary losses. Figure \ref{fig:lambda_ablation} exhibits the significance of these terms on the resulting fine-tuned model's behavior. When holding $\lambda_{cs}=1.0$ and varying $\lambda_{rm}$, perplexity is largely stable for moderate values $\lambda_{rm}\in\{0.01,0.1,1.0\}$  while transfers rise slightly from $278$ to $314$ transfers per layer. When holding $\lambda_{rm}=1.0$ and varying $\lambda_{cs}$, increasing $\lambda_{cs}$ predictably reduces transfers, confirming that the transfer penalty stabilizes expert residency. However, pushing $\lambda_{cs}$ too high sharply harms the fine-tuned model's quality, indicating that aggressively minimizing transfers can over-constrain routing and degrade performance. Figure \ref{fig:lambda_ablation} highlights that moderate $\lambda_{cs}$ and $\lambda_{rm}$ is an effective balancing point between the objectives, simultaneously achieving large expert transfer reductions while minimally impacting perplexity. 
\vspace{0.8\baselineskip}
\paragraph{Effect of Batch Size.}
We study the impact of batched decoding in Figure~\ref{fig:batch_size}. We do not include prior offloading baselines as they do not specifically test regimes with batch size $>1$. For \AlgName, the activation predictor pools the most likely experts across all sequences in the batch, while all other aspects of the caching policy remain identical to the batch size $=1$ setting. \AlgName\ yields significant throughput gains over the base model and maintains increasing throughput as batch size grows. The relative speedups modestly diminish at larger batch sizes as sequence diversity increases the union of requested experts and induces additional transfers. Fine-tuning to amplify prompt-specific expert preferences remains valuable for memory-constrained MoE inference in multi-request settings.
\vspace{0.8\baselineskip}
\paragraph{Effect of Output Generation Length.}
A key observation is that fine-tuning does not over-constrain the model and lead to larger performance drops over longer generations. In Table \ref{tab:output_quality}, we present the performance of the fine-tuned model at different generation lengths for all three model architectures. Across various output token lengths, the fine-tuned model remains performant. This suggests that \AlgName\ preserves output quality as generation length increases, rather than trade long-horizon stability for short-context gains. Importantly, this supports the view that the cache simulation loss remains viable for variable-length generations.
\vspace{0.8\baselineskip}
\paragraph{Coupling Fine-Tuning with Previous Baselines.} Since the fine-tuning procedure in \AlgName\ is not contingent on specific prefetching, quantization, or sparsity schemes, the fine-tuned model checkpoint can be used as a stand-in for the base model when using prior baselines. Table \ref{tab:baseline_ablation_both} shows that the fine-tuned model checkpoint can improve the throughput of cache-based methods like FLoE and Mixtral-Offloading. Using the same VRAM restrictions as Section \ref{sec:main}, swapping the base MoE model for the fine-tuned version can yield improvements of up to $5.49$ tokens/s. Therefore, the fine-tuning procedure proposed in \AlgName\ can be used to augment prior and future offloading baselines. 
\begin{table}[t]
\centering
\footnotesize
\setlength{\tabcolsep}{3.5pt}
\renewcommand{\arraystretch}{1.15}
\caption{Fine-tuned model perplexity across generation lengths. The quality of the fine-tuned model does not degrade with long output horizons.}
\label{tab:output_quality}

\newcolumntype{Y}{>{\centering\arraybackslash}X}

\begin{tabularx}{\linewidth}{@{}>{\raggedright\arraybackslash}p{0.32\linewidth} *{3}{Y}@{}}
\specialrule{0.09em}{0.2em}{0.2em}
\rowcolor{gray!12}
\textbf{Output Length} &
\textbf{OLMoE} &
\textbf{Phi-3.5-MoE} &
\textbf{Mixtral-8x7B} \\
\specialrule{0.06em}{0.2em}{0.2em}
\textbf{64 Tokens} & $13.11$ & $5.56$ & $6.12$ \\
\textbf{128 Tokens} & $10.42$ & $4.81$ & $4.80$ \\
\textbf{256 Tokens} & $9.50$ & $4.18$ & $4.14$ \\
\textbf{512 Tokens} & $10.56$ & $3.95$ & $3.92$ \\
\textbf{1024 Tokens} & $9.86$ & $4.38$ & $4.29$ \\
\specialrule{0.09em}{0.2em}{0.2em}
\end{tabularx}
\end{table}

\begin{table}[t]
\centering
\footnotesize
\setlength{\tabcolsep}{2.2pt}
\renewcommand{\arraystretch}{1.12}
\caption{Impact of fine-tuning on prior baselines. Each entry reports throughput (tokens/s).}
\label{tab:baseline_ablation_both}

\newcolumntype{Y}{>{\centering\arraybackslash}X}

\begin{tabularx}{\linewidth}{@{}>{\raggedright\arraybackslash}p{0.30\linewidth}@{} *{4}{Y}@{}}
\specialrule{0.09em}{0.2em}{0.2em}
\rowcolor{gray!12}
&
\multicolumn{2}{c}{\textbf{Dataset: Dolly15K}} &
\multicolumn{2}{c}{\textbf{Dataset: GSM8K}} \\
\specialrule{0.06em}{0.15em}{0.15em}
\rowcolor{gray!6}
\textbf{Method} &
\textbf{OLMoE} & \textbf{Phi-3.5-MoE} &
\textbf{OLMoE} & \textbf{Phi-3.5-MoE} \\
\specialrule{0.06em}{0.2em}{0.2em}

\textbf{FLoE} &
$14.19$ & $5.23$ &
$14.62$ & $5.61$ \\

\textbf{\quad + Fine-Tuning} &
$20.11$ & $6.28$ &
$20.24$ & $9.56$ \\

\textbf{Mixtral-Offloading} &
$5.70$ & $8.58$ &
$5.55$ & $8.52$ \\

\textbf{\quad + Fine-Tuning} &
$9.22$ & $8.96$ &
$9.07$ & $8.91$ \\

\specialrule{0.09em}{0.2em}{0.2em}
\end{tabularx}
\vspace{-2mm}
\end{table}
\section{Conclusion}
\label{sec:conclusion}
We present \AlgName, a procedure that makes MoE models more deployment-friendly under tight VRAM budgets. \AlgName\ fine-tunes an MoE model to strongly prefer a small subset of experts on a per-sequence basis. With these stronger and more consistent expert preferences, caching the preferred experts in GPU memory yields substantially fewer expert transfers during decoding, improving throughput without sacrificing output quality on downstream tasks.
\\ \\
We identify three areas for future work. First, we aim to conduct studies with larger-scale fine-tuning on general-purpose corpora. While we evaluate generalization (Appendix \ref{subsec:generalization}), future work should explore whether this approach is valuable for deploying MoE models across diverse tasks. Second, our current design uses the same number of cached experts per layer, whereas layer-wise cache budgets may provide increased flexibility. Finally, for very long generations, dynamically adapting cache sizes over time may be beneficial. Overall, we view \AlgName\ as a step towards deployment-ready MoE models that retain strong quality while operating efficiently under real hardware constraints.
\section*{Acknowledgments}

This work was partially supported by NSF grants CCF 2045694, CCF 2428569, CNS-2112471, CPS-2111751, ONR grant N00014-23-1-2149, and an AI2C Seed grant. This work used Bridges-2 GPU at the Pittsburgh Supercomputing Center through allocations CIS250149 and CIS250011 from the Advanced Cyberinfrastructure Coordination Ecosystem: Services \& Support (ACCESS) program, which is supported by NSF grants \#2138259, \#2138286, \#2138307, \#2137603, and \#2138296.  \citep{access}. We would like to thank Aneesha Sampath, Divyansh Jhunjhunwala, Tim Dettmers, and Jianyu Wang for
providing feedback and contributing to discussions for the project.
\bibliographystyle{plainnat}
\bibliography{refs}
\onecolumn
\appendix
\section*{Appendix Contents}
\label{sec:appendix}
{\color{blue}
\noindent
\textbf{A \hspace{0.5em} Related Works} \dotfill \textbf{15}
\\[0.5em]
\textbf{B \hspace{0.5em} Additional Experimental Details} \dotfill \textbf{16} \\
\indent \hspace{1.5em} B.1 \hspace{0.5em} Model Details \dotfill 16 \\
\indent \hspace{1.5em} B.2 \hspace{0.5em} Optimizer and Loss Function Parameters \dotfill 16 \\
\indent \hspace{1.5em} B.3 \hspace{0.5em} Hardware Configurations \dotfill 117 \\ 
\indent \hspace{1.5em} B.4 \hspace{0.5em} Evaluation Settings \dotfill 17 \\ [0.5em]
\textbf{C \hspace{0.5em} Loss Function Design Justification} \dotfill \textbf{18} \\
\indent \hspace{1.5em} C.1 \hspace{0.5em} Choice of $\mathcal{L}_{cs}$ \dotfill 18 \\ 
\indent \hspace{1.5em} C.2 \hspace{0.5em} Choice of $\mathcal{L}_{rm}$ \dotfill 20 \\ [0.5em]
\textbf{D \hspace{0.5em} Additional Experiments} \dotfill \textbf{22} \\
\indent \hspace{1.5em} D.1 \hspace{0.5em} Out-of-Distribution Generalization Performance of \AlgName \dotfill 22 \\ 
\indent \hspace{1.5em} D.2 \hspace{0.5em} Effect of Output Generation Length on Throughput \dotfill 22 \\ 
\indent \hspace{1.5em} D.3 \hspace{0.5em} Impact of Fine-Tuning on Expert Routing \dotfill 23 \\ 
\indent \hspace{1.5em} D.4 \hspace{0.5em} Ablation on GPU VRAM Budget \dotfill 27 \\ 
\indent \hspace{1.5em} D.5 \hspace{0.5em} Ablation on Quantized Experts \dotfill 27 \\ 
\indent \hspace{1.5em} D.6 \hspace{0.5em} Ablation on Soft Cache Capacity in Loss \dotfill 28 \\
\indent \hspace{1.5em} D.7 \hspace{0.5em} Ablation on Loss Function Decay Factor $\gamma$ \dotfill 28 \\ 
\indent \hspace{1.5em} D.8 \hspace{0.5em} Ablation on Cache Eviction Policy \dotfill 28 \\ 
} 

\newpage 
\section{Related Works}
\label{sec:related}
\paragraph{Mixture-of-Experts Model Architectures.} Mixture-of-Experts (MoE) model architectures were first proposed as an alternative to traditional transformer-based models \cite{shazeer2017}. In MoE models, the Feed Forward Network (FFN) present in each transformer block is separated into multiple FFNs which are sparsely activated on a per-token basis. These sparse activations allow MoE models to increase model capacity without a proportional increase in training compute, since only a small subset of experts is executed per token \cite{lepikhin2021gshard, 10.5555/3586589.3586709, pmlr-v162-du22c, 10.5555/3540261.3540918}. A few design choices are critical for making MoE models effective in practice. First, expert granularity controls the size and number of experts in the MoE model. Coarse-grained MoE models such as Mixtral-8x7B \cite{jiang2024mixtral}, Phi-3.5-MoE \cite{abdin2024phi3technicalreporthighly}, and DBRX \cite{mosaic2024dbrx} all use a relatively small number of large experts. In contrast, fine-grained MoE architectures such as OLMoE \cite{muennighoff2025olmoe}, Qwen2MoE \cite{yang2024qwen2technicalreport}, Qwen3MoE \cite{yang2025qwen3technicalreport}, DeepSeek-V2 \cite{deepseekai2024deepseekv2strongeconomicalefficient}, and DeepSeek-V3 \cite{deepseekai2025deepseekv3technicalreport} use a higher number of small experts to increase expert specialization and reduce training costs. Another axis by which MoE models differ is their token routing strategy. Initial MoE model architectures such as GShard \cite{lepikhin2021gshard} and GLaM \cite{pmlr-v162-du22c} employ router load balancing losses to ensure all experts are utilized in training to prevent expert collapse and token dropping from overloaded experts. BASE Layers \cite{pmlr-v139-lewis21a} instead formulates routing as a linear assignment problem, ensuring experts receive an equal number of tokens during generation. Expert Choice Routing \cite{10.5555/3600270.3600785} allows each expert to instead pick its Top-$K$ tokens. More recent works such as Loss-Free Balancing \cite{wang2024auxiliarylossfreeloadbalancingstrategy} add expert-wise biases to the router's decisions and update these biases throughout generation to achieve approximate load balancing without an auxiliary loss term. Granularity and routing are critical design choices in downstream deployments as they directly shape the tradeoff between capacity, efficiency, and serving latency.
\vspace{0.8\baselineskip}
\paragraph{On-Device Mixture-of-Experts.} The proliferation of open-source MoE models has made on-device deployments of MoEs a critical research direction. MoE inference is often memory-bound since all experts must be accessible at inference time despite their sparse activations. In turn, on-device settings are particularly challenging for MoEs since these devices often have limited GPU memory. In this regime, expert offloading techniques have emerged as a promising way of managing the memory constraints imposed by MoE models. Edge-MoE \cite{10906629} assumes a highly constrained setting that requires offloading quantized experts to disk memory. However, the following works allow for a larger amount of GPU VRAM and enough CPU DRAM to host all offloaded model weights. Mixtral-Offloading \cite{eliseev2023fastinferencemixtureofexpertslanguage}, which operates in this more lenient regime, uniformly quantizes experts to $3$ bits and all remaining weights to $4$ bits, allowing for a larger number of experts to remain resident in GPU DRAM. This variable bit-width quantization incurs additional compute overhead and a non-negligible decrease in model performance. FLoE \cite{zhou2025floe} similarly uses quantization and activation sparsity to increase the size of the GPU-resident expert set but applies them to specific layers where they find compression techniques have the least effect on performance. Other techniques avoid lossy compression schemes and exploit other avenues for reducing I/O latency. MoE-Infinity \cite{xue2025moeinfinityefficientmoeinference} prefetches experts by profiling previous expert activations and uses $K$-means to identify experts that are likely to be activated for a specific request. FineMoE \cite{yu2025taminglatencymemorytradeoffmoebased} aims to improve on this prefetching scheme by creating expert maps that track token-level activation patterns to prefetch experts in every iteration. An alternative approach proposed by Fiddler \cite{kamahori2025fiddler} is to directly move intermediate activations to CPU memory and perform expert computation on the CPU for CPU-resident experts. As activations are smaller than the expert weights themselves, this method trades compute efficiency for a reduction in I/O latency. Methods such as HarMoEny \cite{doucet2025harmoenyefficientmultigpuinference} and Faster MoE \cite{yang2025fastermoellminference} which target the multi-GPU inference setting aim to optimize communication but are not directly applicable to this resource-constrained setting. 
\vspace{0.8\baselineskip}
\paragraph{Routing and Semantic Locality.} While expert offloading is a promising direction for reducing the memory demands of MoE models, recent works suggest that naïvely applying expert offloading may lead to substantial hidden I/O latency. Concretely, MoE models may exhibit low local routing consistency \cite{liang2025modelssuitexpertoffloading}, where nearby tokens may activate different experts. This diversity in activations increases the number of cache transfers, reducing the efficacy of a GPU-resident expert set. Recent literature approaches the issue of routing locality in various ways. Oracle-MoE \cite{zhou2025oraclemoe} pretrains MoE architectures to route tokens based on attention score-based semantic similarity in order to reduce expert transfers. However, its evaluation is limited to GPT-2-scale models, leaving open how well the approach extends to modern MoE models, where larger architectures and tighter GPU memory budgets can make offloading and transfer bottlenecks more pronounced. BlockFFN \cite{song2025blockffnendsideaccelerationfriendlymixtureofexperts} similarly performs MoE pretraining but utilizes an auxiliary loss function that encourages neighboring tokens to activate similar experts. BlockFFN does not specifically target expert offloading and only mentions speculative decoding as a device-side acceleration technique that composes with their method. 
\newpage
\section{Additional Experimental Details}
\label{sec:details}

\subsection{Model Details}
In Table \ref{tab:moe_backbone_details}, we include details of OLMoE, Phi-3.5-MoE, and Mixtral-8x7B, the models that we use in all experiments. 

\begin{table}[htbp]
\centering
\footnotesize
\setlength{\tabcolsep}{3.6pt}
\renewcommand{\arraystretch}{1.12}
\caption{Model details for the MoE backbones used in all experiments. Total parameters count all experts; active parameters correspond to the parameters used per token under Top-$K$ routing. FP16 size is the approximate weight footprint assuming $2$ bytes/parameter.}
\label{tab:moe_backbone_details}

\begin{tabularx}{\linewidth}{@{}l *{7}{>{\centering\arraybackslash}X}@{}}
\specialrule{0.09em}{0.2em}{0.2em}
\rowcolor{gray!12}
\textbf{Model} &
\textbf{Layers} &
\makecell{\textbf{Experts/}\\\textbf{Layer}} &
\makecell{\textbf{Active}\\\textbf{Experts}} &
\makecell{\textbf{Total}\\\textbf{Params. (B)}} &
\makecell{\textbf{Active}\\\textbf{Params. (B)}} &
\makecell{\textbf{Context}\\\textbf{Len.}} &
\makecell{\textbf{FP16}\\\textbf{Size (GB)}} \\
\specialrule{0.06em}{0.15em}{0.15em}

OLMoE & $16$ & $64$ & $8$ & $6.9$ & $1.3$ & $4$K & $13.8$ \\
Phi-3.5-MoE & $32$ & $16$ & $2$ & $42.0$ & $6.6$ & $128$K & $84.0$ \\
Mixtral-8x7B & $32$ & $8$ & $2$ & $46.7$ & $12.9$ & $32$K & $93.4$ \\

\specialrule{0.09em}{0.2em}{0.2em}
\end{tabularx}
\end{table}

\subsection{Optimizer and Loss Function Parameters}
\label{subsec:opt}
In Tables \ref{tab:finetune_hparams_full} and \ref{tab:mlp_hparams}, we highlight training hyperparameters used for fine-tuning each MoE model and training the MLP activation predictor, respectively. 
\begin{table}[htbp]
\centering
\footnotesize
\setlength{\tabcolsep}{3.6pt}
\renewcommand{\arraystretch}{1.12}
\caption{Fine-tuning hyperparameters for each dataset and MoE model backbone.}
\label{tab:finetune_hparams_full}

\begin{tabularx}{\linewidth}{@{}l *{6}{>{\centering\arraybackslash}X}@{}}
\specialrule{0.09em}{0.2em}{0.2em}
\rowcolor{gray!12}
&
\multicolumn{3}{c}{\textbf{Dataset: Dolly15K}} &
\multicolumn{3}{c}{\textbf{Dataset: GSM8K}} \\
\specialrule{0.06em}{0.15em}{0.15em}
\rowcolor{gray!6}
\textbf{Hyperparameter} &
\textbf{OLMoE} & \textbf{Phi-3.5-MoE} & \textbf{Mixtral-8x7B} &
\textbf{OLMoE} & \textbf{Phi-3.5-MoE} & \textbf{Mixtral-8x7B} \\
\specialrule{0.06em}{0.15em}{0.15em}

Peak Learning Rate & $10^{-5}$ & $10^{-5}$ & $10^{-5}$ & $10^{-5}$ & $10^{-5}$ & $10^{-5}$ \\
Optimizer & AdamW & AdamW & AdamW & AdamW & AdamW & AdamW \\
Schedule & Linear & Linear & Linear & Linear & Linear & Linear \\ 
Warmup Ratio & $0.03$ & $0.03$ & $0.03$ & $0.03$ & $0.03$ & $0.03$ \\
Max Tokens & $512$ & $512$ & $512$ & $512$ & $512$ & $512$ \\
LoRA Rank ($r$) & $32$ & $32$ & $32$ & $32$ & $32$ & $32$ \\
LoRA $\alpha$ & $16$ & $16$ & $16$ & $16$ & $16$ & $16$ \\

Epochs & $3$ & $3$ & $3$ & $5$ & $5$ & $5$ \\
$\lambda_{\text{cs}}$ & $0.5$ & $0.5$ & $0.5$ & $0.05$ & $0.05$ & $0.05$ \\
$\lambda_{\text{rm}}$ & $0.1$ & $0.1$ & $0.1$ & $0.01$ & $0.01$ & $0.01$ \\
$\gamma$ (Cache Decay) & $0.9$ & $0.9$ & $0.9$ & $0.9$ & $0.9$ & $0.9$ \\
$\rho$ (Rank Margin) & $0.1$ & $0.1$ & $0.1$ & $0.1$ & $0.1$ & $0.1$ \\

Cache Capacity $C$ & $16$ & $4$ & $2$ & $16$ & $4$ & $2$ \\

\specialrule{0.09em}{0.2em}{0.2em}
\end{tabularx}
\end{table}

\begin{table}[htbp]
\centering
\footnotesize
\setlength{\tabcolsep}{6pt}
\renewcommand{\arraystretch}{1.05}
\caption{Activation predictor MLP training hyperparameters (same values across datasets).}
\label{tab:mlp_hparams}

\begin{tabularx}{0.55\linewidth}{@{}l >{\centering\arraybackslash}X@{}}
\toprule
\rowcolor{gray!12}
\textbf{Hyperparameter} & \textbf{Value} \\
\midrule
Number of Layers & $2$ \\
Input Size & $768$ \\
Hidden Dimension & $1024$ \\
Loss & KL Divergence \\
Learning Rate & $2\times 10^{-4}$ \\
Epochs & $10$ \\
Optimizer & SGD \\
Momentum & $0.9$ \\
Batch Size & $16$ \\
\bottomrule
\end{tabularx}
\end{table}

\newpage 
\subsection{Hardware Configurations}
In Table \ref{tab:hardware_configs_compact}, we characterize the various hardware setups and configurations used in our experiments. 
\begin{table}[htbp]
\centering
\footnotesize
\setlength{\tabcolsep}{4.5pt}
\renewcommand{\arraystretch}{1.05}
\caption{Hardware configurations used in offloaded MoE inference experiments.}
\label{tab:hardware_configs_compact}

\begin{tabularx}{\linewidth}{@{}l *{3}{>{\centering\arraybackslash}X}@{}}
\toprule
\rowcolor{gray!12}
\textbf{Spec} & \textbf{H100 Setup} & \textbf{A100 Setup} & \textbf{RTX 4090 Setup} \\
\midrule
GPU & H100 & A100 & RTX 4090 \\
GPU VRAM (GB) & $80$ & $40$ & $24$ \\
Host CPU & Intel Xeon Platinum 8470 & AMD EPYC 7J13 & AMD EPYC 7B13 \\
Host DRAM (GB) & $2048$ & $216$ & $270$ \\
PCIe Gen xLanes & Gen $5$ x$16$  & Gen $4$ x$16$ & Gen $4$ x$16$ \\
PCIe Bandwidth (GB/s) & $64.0$  & $32.0$ & $32.0$ \\
\bottomrule
\end{tabularx}
\end{table}

\subsection{Evaluation Settings}
Finally, in Table \ref{tab:eval_config}, we present model settings and decoding metrics used during the evaluation of our method. 
\begin{table}[htbp]
\centering
\footnotesize
\setlength{\tabcolsep}{4.8pt}
\renewcommand{\arraystretch}{1.06}
\caption{Evaluation configuration used for throughput measurements.}
\label{tab:eval_config}

\begin{tabularx}{\linewidth}{@{}l *{3}{>{\centering\arraybackslash}X}@{}}
\toprule
\rowcolor{gray!12}
 & \textbf{OLMoE} & \textbf{Phi-3.5-MoE} & \textbf{Mixtral-8x7B} \\
\midrule
Resident Experts / Layer & $16$ & $8$ & $5$ \\
Quantized Modules & $\texttt{gate\_proj, up\_proj, down\_proj}$ & $\texttt{w1, w2, w3}$ & $\texttt{w1, w2, w3}$ \\
Decoding Strategy & Greedy & Greedy & Greedy \\
Throughput Metric & Output tokens/s & Output tokens/s & Output tokens/s \\
Max Output Tokens & $256$ & $256$ & $256$ \\
\bottomrule
\end{tabularx}
\end{table}
\newpage
\section{Loss Function Design Justification}
\label{sec:dc}
In this section, we provide theoretical justification for the choice of our loss functions. Throughout the section, we will use $b$ and $f$ in the subscript to associate symbols with the base model and the fine-tuned model respectively, $\x = [\x_1, \ldots, \x_\T]$ refers to the sequence of tokens, and $\textup{Top-}\C(\mathbf{\cdot})$ refers to an operation that returns a vector with $L_1$ norm $C$ and entries $1$ for $\C$ positions with the highest magnitude and $0$ otherwise. $[N]$ for any natural number $N\in\mathbb{Z}_+$ denotes the set $[1,\ldots, N]$.
\subsection{Choice of $\lcs$}
The cache simulation loss $\lcs(\x)$ is designed to penalize expert activation patterns with excessive switching. Previous work indicates that certain MoE models activate similar experts across adjacent tokens, while others consistently activate a small subset of experts across the decoding trajectory for a given prompt. In the former case, an LRU cache eviction policy is preferable, whereas in the latter an LFU-based eviction policy is more appropriate under conventional caching frameworks.
\\ \\ 
By fine-tuning the router layers, we can control the underlying expert activation distribution, allowing us to tailor the distribution towards a specific cache eviction policy. The $\lcs$ loss encourages tokens to have similar activation distributions. The decay parameter $\gm$ weights older requests relative to newer ones when computing this similarity. In the following proposition, we will first explain how $\gm$ can be used to interpolate between LFU and LRU updates with hard cache states.
\begin{definition}[$\gm$ cache eviction]
\label{def:gm}
    For any given sequence of tokens $\x = [\x_1, \ldots, \x_T]$, let $\rv_f^{(\ly,\ti)}(\x)\in \{0,1\}^E$ denote the binary request vector at layer $\ly$ and token position $\ti$ with $\|\rv_f^{(\ly,\ti)}(\x)\|_1 = \K$. Then, we define the $\gm$-discounted count and the $\gm$-discounted cache state under model $f$ ($\Ch_{f,\gm}$) as follows 
    \begin{align}
        \Ct_{f,\gm}^{(\ly,\ti+1)}(\x) &= \gm\Ct_{f,\gm}^{(\ly,\ti)}(\x) + \rv_f^{(\ly,\ti)}(\x),\label{eq:countdef}\\
        \Ch_{f,\gm}^{(\ly,\ti+1)}(\x) &= \textup{Top-}\C\left(\Ct_{f.\gm}^{(\ly,\ti+1)}(\x)\right),
    \end{align}
    where $\Ct_{f,\gm}^{(\ly,\ti)}(\x)$ is a vector representing the discounted expert request numbers up to $t-1$ for layer $\ly$ under sequence $\x$ and $\Ct_{f,\gm}^{(\ly,1)}(\x)$ is initialized as a vector with $L_1$ norm $\C$. Note that $\Ct_{f,\gm}^{(\ly,1)}$ can be set to anything based on the prefetch initialization used. Alternatively, one could also use a uniform initialization.
\end{definition}
\begin{remark}
From Definition \ref{def:gm}, we have
\begin{itemize}[leftmargin=*,itemsep=0pt,topsep=0pt]
    \item Using $\gm =1$  corresponds to LFU cache updates, and when $\gm\rightarrow 0^+$ we get LRU cache updates. 
\item For any $\gm$, the cache can be updated lazily. Specifically, any expert that is present in the cache for layer $\ly$ at time $\ti+1$, that is any index with a nonzero entry in $\Ch_{f,\gm}^{(\ly,\ti+1)}(\x)$, must either have been present in $\Ch_{f,\gm}^{(\ly,\ti)}(\x)$ or appear in $\rv_f^{(\ly,\ti)}(\x)$. Consequently, the cache can be updated during offloaded inference without incurring additional overhead.
\end{itemize}
\end{remark}
\noindent The total number of cache misses equals the number of PCIe transfers. Under a $\gm$-cache eviction policy, the number of cache misses at layer $\ly$ over $T$ tokens is given by
\begin{equation*}
\textup{Cache misses at layer } \ly :=\sum_{\ti=1}^\T\left\langle\rv_f^{(\ly,\ti)}(\x), \mathbf{1}- \Ch_{f,\gm}^{(\ly,\ti)}(\x) \right\rangle,
\end{equation*}
and the hard loss $\lcs^{\textup{hard}}(\x,f)$ as
\begin{equation}
    \lcs^{\textup{hard}}(\x,f) := \frac{1}{\Ly\T}\sum_{\ly=1}^{\Ly}\sum_{\ti=1}^{\T}\left\langle\rv_f^{(\ly,\ti)}(\x), \mathbf{1}- \Ch_{f,\gm}^{(\ly,\ti)}(\x) \right\rangle.
    \label{eq:lcshard}
\end{equation}
However, the loss above is hard to differentiate with respect to the parameters of $f$ due to the  $\textup{Top-}\C$ operation used to define $\Ch_{f,\gm}^{(\ly,\ti)}(\x)$. Hence, we use a ``soft" cache proxy $\cv_{f,\gm}^{(\ly,\ti)}(\x)$ instead of $\Ch_{f,\gm}^{(\ly,\ti)}(\x)$ where
\begin{equation}
    \cv_{f,\gm}^{(\ly,\ti)}(\x) := \frac{\Ct_{f,\gm}^{(\ly,\ti)}(\x)}{\left\|\Ct_{f,\gm}^{(\ly,\ti)}(\x)\right\|_1} \cdot \C,
    \label{eq:cachedef}
\end{equation}
which is the $\Ct_{f,\gm}^{(\ly,\ti)}(\x)$ scaled to make its $L_1$ norm equal to $\C$.
\begin{proposition}
\label{prop:yelab}
 Using the definitions in Equations \eqref{eq:countdef} and \eqref{eq:cachedef}, we have
    \begin{align*}
    \left\|\Ct_{f,\gm}^{(\ly,\ti)}(\x)\right\|_1 = \gm^{t-1}\C + K\sum_{i=1}^{t-1} \gm^{t-i-1},
    \end{align*}
 define $\Gm^{(t)} = \gm^{t-1} +\frac{\K}{\C}\sum_{i=1}^{t-1}\gm^{t-i-1}$. Then, by Equation \eqref{eq:cachedef}, one can recursively update the soft cache state as 
\begin{equation*}
    \cv_{f,\gm}^{(\ly,\ti+1)}(\x) := \frac{\gm\Gm^{(t)} \cv_{f,\gm}^{(\ly,\ti)}(\x)+ \rv_{f}^{(\ly,\ti)}(\x)}{\Gm^{(t+1)}} \qquad\textup{and} \qquad \Gm^{(t+1)} = \gm \Gm^{(t)} + \frac{\K}{\C}
\end{equation*}
\end{proposition}
\begin{proof}
Using the definition in Equation \eqref{eq:countdef}, we have
\begin{equation*}
    \left\|\Ct_{f,\gm}^{(\ly,\ti)}(\x)\right\|_1 = \left\|\gm^{t-1}\Ct_{f,\gm}^{(\ly,1)}(\x) + \sum_{i=1}^{t-1} \gm^{t-i-1}\rv_{f}^{(\ly,\ti)}(\x)\right\|_1 =\gm^{t-1}\C + K\sum_{i=1}^{t-1} \gm^{t-i-1}
\end{equation*}
Next, using this, one can show $\forall \ti\in [\T]$
\begin{equation*}
    \cv_{f,\gm}^{(\ly,\ti)}(\x) = \frac{\Ct_{f,\gm}^{(\ly,\ti)}(\x)}{\left\|\Ct_{f,\gm}^{(\ly,\ti)}(\x)\right\|_1} \cdot \C = \frac{\Ct_{f,\gm}^{(\ly,\ti)}(\x)}{\Gm^{(t)}}.
\end{equation*}
and 
\begin{equation*}
    \cv_{f,\gm}^{(\ly,\ti+1)}(\x) = \left(\Gm^{(t+1)}\right)^{-1}\left(\gm \Ct_{f,\gm}^{(\ly,\ti)}(\x)+ \rv_{f}^{(\ly,\ti)}(\x) \right) =\left(\Gm^{(t+1)}\right)^{-1}\left(\gm\Gm^{(t)} \cv_{f,\gm}^{(\ly,\ti)}(\x)+ \rv_{f}^{(\ly,\ti)}(\x) \right). 
\end{equation*}
The recursion for $\Gm^{(t)}$ follows directly from the definition.
\end{proof}
\noindent Subsequently, we define $\lcs(\x,f)$ by replacing the cache state $\Ch_{f,\gm}^{(\ly,\ti)}(\x)$ in Equation \eqref{eq:lcshard} with the ``soft" cache state $\cv_{f,\gm}^{(\ly,\ti)}(\x)$ in order to facilitate fine-tuning the model $f$ via gradient-based methods
\begin{equation*}
    \lcs(\x,f) := \frac{1}{\Ly\T}\sum_{\ly=1}^{\Ly}\sum_{\ti=1}^{\T}\left\langle\rv_f^{(\ly,\ti)}(\x), \mathbf{1}- \cv_{f,\gm}^{(\ly,\ti)}(\x) \right\rangle.
\end{equation*}
The next lemma provides an interpretation of the cache simulation loss $\lcs$. 

\begin{lemma}
\label{lemma:lcsexp}
    Let $\x\sim\mathcal{D}_{\text{train}}$ denote the training dataset of SFT traces. Define $\lcs(f) := \mathop{\mathbb{E}}_{\x\sim\mathcal{D}_{\text{train}}}[\lcs(\x,f)] $, then we have
    \begin{equation*}
        \lcs(f) = K - \frac{1}{\Ly\T}\sum_{\ly=1}^{\Ly}\sum_{\ti=1}^{\T} \left(\Gm^{(\ti)}\right)^{-1}\left(\sum_{i=1}^{t-1}\gm^{t-1-i}\phi_f^{(\ly)}(\ti,i) + \gm^{t-1}\phi_f^{(\ly)}(\ti,1)\right)
    \end{equation*}
\end{lemma}
\noindent where $\phi_f^{(\ly)}(\ti,i) := \mathop{\mathbb{E}}_{\x\sim\mathcal{D}_{\text{train}}}\left[\left\langle\rv_f^{(\ly,\ti)}(\x),\rv_f^{(\ly,i)}(\x)\right\rangle\right]$, $\phi_f^{(\ly)}(t,1) :=\mathop{\mathbb{E}}_{\x\sim\mathcal{D}_{\text{train}}}\left[\left\langle\rv_f^{(\ly,\ti)}(\x),\cv_{f,\gm}^{(\ly,1)}(\x)\right\rangle\right]$, and its derivative with respect to the parameter $\gm$: $\frac{d \lcs(f)}{d\gm}\leq0$ 

\begin{proof}
Recall, $\lcs (\x,f)$ is given by
    \begin{equation*}
        \lcs(\x,f) = \frac{1}{\Ly\T}\sum_{\ly=1}^{\Ly}\sum_{\ti=1}^{\T} \left\langle\rv_f^{(\ly,\ti)}(\x), \one - \cv_{f,\gm}^{(\ly,\ti)}(\x) \right\rangle
    \end{equation*}
    Let us denote $ \lcs^{(\ly,\ti)}(\x):=\left\langle\rv_f^{(\ly,\ti)}(\x), \one - \cv_{f,\gm}^{(\ly,\ti)}(\x) \right\rangle$. Then, for some  initialization $\cv_{f,\gm}^{(\ly,1)}(\x)$, using Proposition \ref{prop:yelab}, one can unroll $\cv_{f,\gm}^{(\ly,\ti)}(\x)$ as follows
    \begin{equation*}
        \cv_{f,\gm}^{(\ly,\ti)}(\x) = \left(\Gm^{(\ti)}\right)^{-1}\left(\sum_{i=1}^{\ti-1}\gm^{t-1-i}\rv_f^{(\ly,i)}(\x) + \gm^{t-1}\cv_{f,\gm}^{(\ly,1)}(\x)\right)\
    \end{equation*}
     Using this expression and the fact $\|\rv_f^{(\ly,\ti)}(\x)\|_1 = K$, one can obtain
    \begin{equation*}
        \lcs^{(\ly,\ti)}(\x,f) = K - \left(\Gm^{(\ti)}\right)^{-1}\left(\sum_{i=1}^{t-1}\gm^{t-1-i}\left\langle\rv_f^{(\ly,\ti)}(\x),\rv_f^{(\ly,i)}(\x)\right\rangle + \gm^{t-1}\left\langle\rv_f^{(\ly,\ti)}(\x),\cv_{f,\gm}^{(\ly,1)}(\x)\right\rangle\right)
    \end{equation*}
    and its derivative with respect to  $\gm$ is given by
    \begin{align*}
        \frac{d \lcs^{(\ly,\ti)}(\x,f)}{d\gm} = -\left(\Gm^{(\ti)}\right)^{-1}\left(\sum_{i=1}^{t-2}(t-1-i)\gm^{t-2-i}\left\langle\rv_f^{(\ly,\ti)}(\x),\rv_f^{(\ly,i)}(\x)\right\rangle + (t-1)\gm^{t-2}\left\langle\rv_f^{(\ly,\ti)}(\x),\cv_{f,\gm}^{(\ly,1)}(\x)\right\rangle\right)\leq 0
    \end{align*}
    Taking the expectation with respect to $\x\sim \mathcal{D}_{\text{train}}$ completes the result.
\end{proof}

\begin{remark}
    From Lemma \ref{lemma:lcsexp}, one can observe that minimizing the loss $\lcs(f)$ entails maximizing the terms $\phi_f^{(\ly)}(\ti,i) := \mathop{\mathbb{E}}_{\x\sim\mathcal{D}_{\text{train}}}\left[\left\langle\rv_f^{(\ly,\ti)}(\x),\rv_f^{(\ly,i)}(\x)\right\rangle\right]$ which encourages the expert request distributions at layer $\ly$ for token positions $\ti$ and $i$ to be similar. Note that $\gm<1$ implies that $\phi_f^{(\ly)}(\ti,i)$ receives a higher weight when $\ti$ and $i$ are closer, and this incentivizes reduced switching. 
\end{remark}

\begin{remark}
    We choose $\gm = 0.9$ in our experiments since $\frac{d \lcs(f)}{d\gm}\leq0$ while also enabling the resulting model to work with both LRU and LFU cache eviction policies.
\end{remark}

\subsection{Choice of $\lrm$}
Given a sequence of tokens $\x = [\x_1, \ldots, \x_T]$, the purpose of the term $\lrm(\x)$ is to discourage router collapse onto a small subset of experts. A broadly adopted strategy for preventing distributional collapse while retaining desirable properties of the base model is KL-regularization. Incorporating a KL penalty is ubiquitous across domains including safety fine-tuning, RLHF, knowledge distillation, and reasoning model training where it serves as a general mechanism for constraining policy updates relative to the base model. In our setting, let $\p^{(\ly,\ti)}_f(\x)$ and $\p^{(\ly,\ti)}_b(\x)$ denote the router distributions at layer $\ly$ and token position $\ti$ induced by the fine-tuned and base routers, respectively. One may then define the KL loss as follows
\begin{equation*}
   \lrm^{\textup{KL}}(\x,f) =  \frac{1}{\Ly\T}\sum_{\ly=1}^{\Ly}\sum_{\ti=1}^{\T} \textup{KL}\left(\p^{(\ly,\ti)}_{f}(\x)\|\p^{(\ly,\ti)}_b(\x)\right).
\end{equation*}
Although this may be a natural choice in tasks such as knowledge distillation where the objective is to match the teacher’s distribution, our setting is different. We are less concerned with divergence between the two distributions in the KL or entropy sense and are ultimately interested in the induced rankings, since the model selects the Top-$\K$ experts from the router distribution. We, therefore, choose a loss that is more directly aligned with mismatches in the induced rankings. Concretely, we model this using the Kendall rank correlation coefficient \cite{kendall1938new}.
\begin{definition}
    Given two router output distributions $\p,\mathbf{q}\in\Delta^{E}$, the Kendall rank correlation coefficient $\tau(\p,\mathbf{q}) \in [-1,1]$ is defined as
    \begin{equation*}
        \tau(\p,\mathbf{q}) = 1- \frac{2\textup{Inv}(\p,\mathbf{q})}{{n\choose 2}},
    \end{equation*}
    where $\textup{Inv}(\p,\mathbf{q})$ is the number of pairwise inversions between distributions $\p$ and $\mathbf{q}$. 
\end{definition}
\noindent Maximizing the rank correlation coefficient is equivalent to minimizing the inversion number. Ideally, one could use the following loss function
\begin{equation*}
    \lrm^{\textup{inv}}(\x) = \frac{1}{\Ly\T}\sum_{\ly=1}^{\Ly}\sum_{\ti=1}^{\T} \textup{Inv}\left(\p^{(\ly,\ti)}_{f}(\x),\p^{(\ly,\ti)}_b(\x)\right).
\end{equation*}
Note that when this loss is $0$, the experts activated by the fine-tuned router and the base router will be exactly the same. However, the term $\textup{Inv}(\cdot,\cdot)$ is difficult to parameterize and differentiate with respect to the parameters of the fine-tuned model $f$. One can express the inversion number as follows
\begin{equation*}
\textup{Inv}\left(\p^{(\ly,\ti)}_{f}(\x),\p^{(\ly,\ti)}_b(\x)\right) = \sum\limits_{i,j\in \left[E\right]}g_b^{(\ly,\ti)}(\x,i,j)\;g_{f,\textup{Inv}}^{(\ly,\ti)}(\x,i,j),
\end{equation*}
where 
\begin{align*}
    g_b^{(\ly,\ti)}(\x,i,j):=\mathbb{I}\!\left\{\p^{(\ly,\ti)}_{b,i}(\x)>\p^{(\ly,\ti)}_{b,j}(\x)\right\} \quad\text{and} \quad g_{f,\textup{Inv}}^{(\ly,\ti)}(\x,i,j) := \mathbb{I}\!\left\{\p^{(\ly,\ti)}_{f,i}(\x)<\p^{(\ly,\ti)}_{f,j}(\x)\right\}.
\end{align*}
\noindent To facilitate learning using gradient-based methods, we use $\mlt$ as a proxy for the inversion number $\textup{Inv}(\cdot,\cdot)$. Specifically, we replace $g_{f,\textup{Inv}}$ by $g_f$, i.e.,
\begin{equation}
\mlt(\x) = \!\!\!\!\! \sum\limits_{i,j\in \left[E\right]}g_b^{(\ly,\ti)}(\x,i,j)\;g_f^{(\ly,\ti)}(\x,i,j) \quad \textup{with}\quad g_f^{(\ly,\ti)}(\x,i,j):=\left[\rho-\left(\p^{(\ly,\ti)}_{f,i}(\x)-\p^{(\ly,\ti)}_{f,j}(\x)\right)\right]_+ 
\label{eq:mltdef}
\end{equation}
and define $\lrm$ as
\begin{equation*}
    \lrm(\x,f) = \frac{1}{\Ly\T}\sum_{\ly=1}^{\Ly}\sum_{\ti=1}^{\T} \mlt(\x).
\end{equation*} 
\begin{lemma}
\label{lemma:lrm}
    Define 
        $\mlt(\x) = \sum\limits_{i,j\in \left[E\right]}\mathbb{I}\!\left\{\p^{(\ly,\ti)}_{b,i}(\x)>\p^{(\ly,\ti)}_{b,j}(\x)\right\}
\left[\rho-\left(\p^{(\ly,\ti)}_{f,i}(\x)-\p^{(\ly,\ti)}_{f,j}(\x)\right)\right]_+ $ as in Equation \eqref{eq:mltdef}. Then, we have $$\mlt(\x)\geq \rho \;\textup{Inv}\left(\p^{(\ly,\ti)}_{f}(\x),\p^{(\ly,\ti)}_b(\x)\right).$$
\end{lemma}
\begin{proof}
    Without loss of generality, assume that no two entries for either vector $\p^{(\ly,\ti)}_{f}(\x)$ or $\p^{(\ly,\ti)}_b(\x)$ are equal. Then, we have 
    \begin{align*}
        \rho \;\textup{Inv}\left(\p^{(\ly,\ti)}_{f}(\x),\p^{(\ly,\ti)}_b(\x)\right) & = \sum\limits_{i,j\in \left[E\right]} \rho \mathbb{I}\!\left\{\p^{(\ly,\ti)}_{f,i}(\x)<\p^{(\ly,\ti)}_{f,j}(\x)\right\}\mathbb{I}\!\left\{\p^{(\ly,\ti)}_{b,i}(\x)>\p^{(\ly,\ti)}_{b,j}(\x)\right\}\\
        &\leq \sum\limits_{i,j\in \left[E\right]}\left[\rho-\left(\p^{(\ly,\ti)}_{f,i}(\x)-\p^{(\ly,\ti)}_{f,j}(\x)\right)\right]_+\mathbb{I}\!\left\{\p^{(\ly,\ti)}_{b,i}(\x)>\p^{(\ly,\ti)}_{b,j}(\x) \right\} = \mlt(\x),
    \end{align*}
    where the last line follows from $\rho\;\mathbb{I}(a<b) \leq [\rho-(a-b)]_+$
\end{proof}
\begin{remark}
    As a consequence of Lemma \ref{lemma:lrm}, minimizing loss $\lrm$ is equivalent to maximizing a lower bound on the rank correlation coefficient between the fine-tuned and base router expert request distributions averaged across layers $\ly \in [\Ly]$ and token positions $\ti \in [\T]$.

\end{remark}

\newpage 
\section{Additional Experiments}
\label{sec:additional}

\subsection{Out-of-Distribution Generalization Performance of \AlgName}
\label{subsec:generalization}
As mentioned in Section \ref{sec:conclusion}, computational constraints prevented us from performing large-scale fine-tuning on a comprehensive, general-purpose dataset. Such training would likely yield a more deployment-friendly model that reduces inference latency across a wider range of downstream tasks. Instead, we evaluate whether \AlgName\ continues to provide throughput improvements when the fine-tuning data used in the pre-deployment stage and the downstream task differ. 
\begin{table}[!htbp]
\centering
\footnotesize
\setlength{\tabcolsep}{3.6pt}
\renewcommand{\arraystretch}{1.12}
\caption{Decoding throughput (tokens/s) when evaluating on Dolly15K vs.\ GSM8K. For \AlgName, we report throughput after fine-tuning on either Dolly15K or GSM8K; baselines are shown for reference. }
\label{tab:cross_eval_throughput}

\begin{tabularx}{\linewidth}{@{}l *{4}{>{\centering\arraybackslash}X}@{}}
\specialrule{0.09em}{0.2em}{0.2em}
\rowcolor{gray!12}
& \multicolumn{2}{c}{\textbf{Eval: Dolly15K}} & \multicolumn{2}{c}{\textbf{Eval: GSM8K}} \\
\specialrule{0.06em}{0.15em}{0.15em}
\rowcolor{gray!6}
\textbf{Method} &
\textbf{Phi-3.5-MoE} & \textbf{Mixtral-8x7B} &
\textbf{Phi-3.5-MoE} & \textbf{Mixtral-8x7B} \\
\specialrule{0.06em}{0.15em}{0.15em}

\AlgName\ \textbf{(Fine-Tune: Dolly15K)} &
$14.34$ & $9.35$ &
$13.52$ & $8.25$ \\

\AlgName\ \textbf{(Fine-Tune: GSM8K)} &
$10.44$ & $8.21$ &
$15.67$ & $10.38$ \\

\specialrule{0.06em}{0.15em}{0.15em}

Fiddler &
$5.88$ & $5.24$ &
$7.26$ & $4.11$ \\

Mixtral-Offloading &
$8.58$ & $5.08$ &
$8.52$ & $5.04$ \\

DeepSpeed-MoE &
$2.63$ & $1.25$ &
$2.69$ & $1.23$ \\

FLoE &
$5.23$ & $2.25$ &
$5.61$ & $2.20$ \\

MoE-Infinity &
$3.73$ & $1.25$ &
$3.79$ & $1.24$ \\

\specialrule{0.09em}{0.2em}{0.2em}
\end{tabularx}
\vspace{-5mm}
\end{table}

\noindent In Table \ref{tab:cross_eval_throughput}, we consider the same settings and resource constraints highlighted in Sections \ref{sec:setup} and \ref{sec:main}. We find that even when the fine-tuning dataset and downstream task differ fundamentally, the throughput improvements relative to prior baselines persist but are partially dampened. For example, when evaluated on a holdout of Dolly15K, Phi-3.5-MoE fine-tuned exclusively on GSM8K still achieves $10.44$ tokens/s, $1.86$ tokens/s greater than Mixtral-Offloading, the next best baseline. However, this is still slower than the $14.34$ tokens/s achieved when Phi-3.5-MoE is fine-tuned on Dolly15K itself. Mixtral-8x7B can remarkably achieve $8.21$ tokens/s when fine-tuned on GSM8K and evaluated on a holdout of Dolly15K, only $1.14$ tokens/s less than when it is fine-tuned on Dolly15K. Similarly, when Mixtral-8x7B is fine-tuned on Dolly15K, its decoding throughput is only $2.13$ tokens/s less on GSM8K then a model fine-tuned on in-distribution data. This suggests that the improvements present in \AlgName\ are largely preserved even when the resource-constrained device in question has a particularly different local data distribution. We hope that future work can better demonstrate \AlgName's generalizability on various downstream tasks by performing more diverse pre-deployment fine-tuning and downstream evaluation. 

\subsection{Effect of Output Generation Length on Throughput}
\begin{figure*}[htbp]
  \centering
  \includegraphics[width=0.92\textwidth]{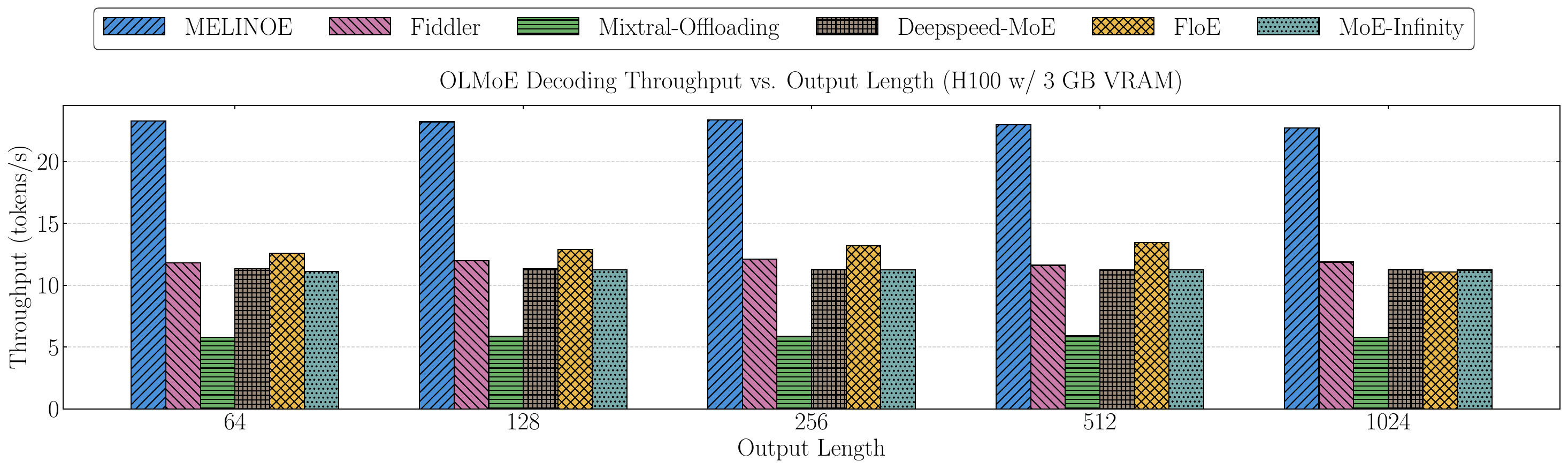}
  \caption{Throughput of baselines with various output lengths using OLMoE on the H100 setup with 3GB of VRAM.}
  \vspace{-2mm}
  \label{fig:output_length}
\end{figure*}
\noindent As output length increases, \AlgName\ maintains stable throughput (Figure \ref{fig:output_length}), which suggests that the fine-tuned model's expert preferences endure even in longer generations. In this regime, memory-efficiency becomes increasingly critical as the KV cache grows with the number of generated tokens, leaving less room for resident experts. Despite this, \AlgName\ sustains near constant tokens/s, indicating that its routing stability reduces cache churn over long decoding horizons. 
\newpage
\subsection{Impact of Fine-Tuning on Expert Routing}
\begin{figure*}[!htbp]
  \centering
  \includegraphics[width=0.95\textwidth]{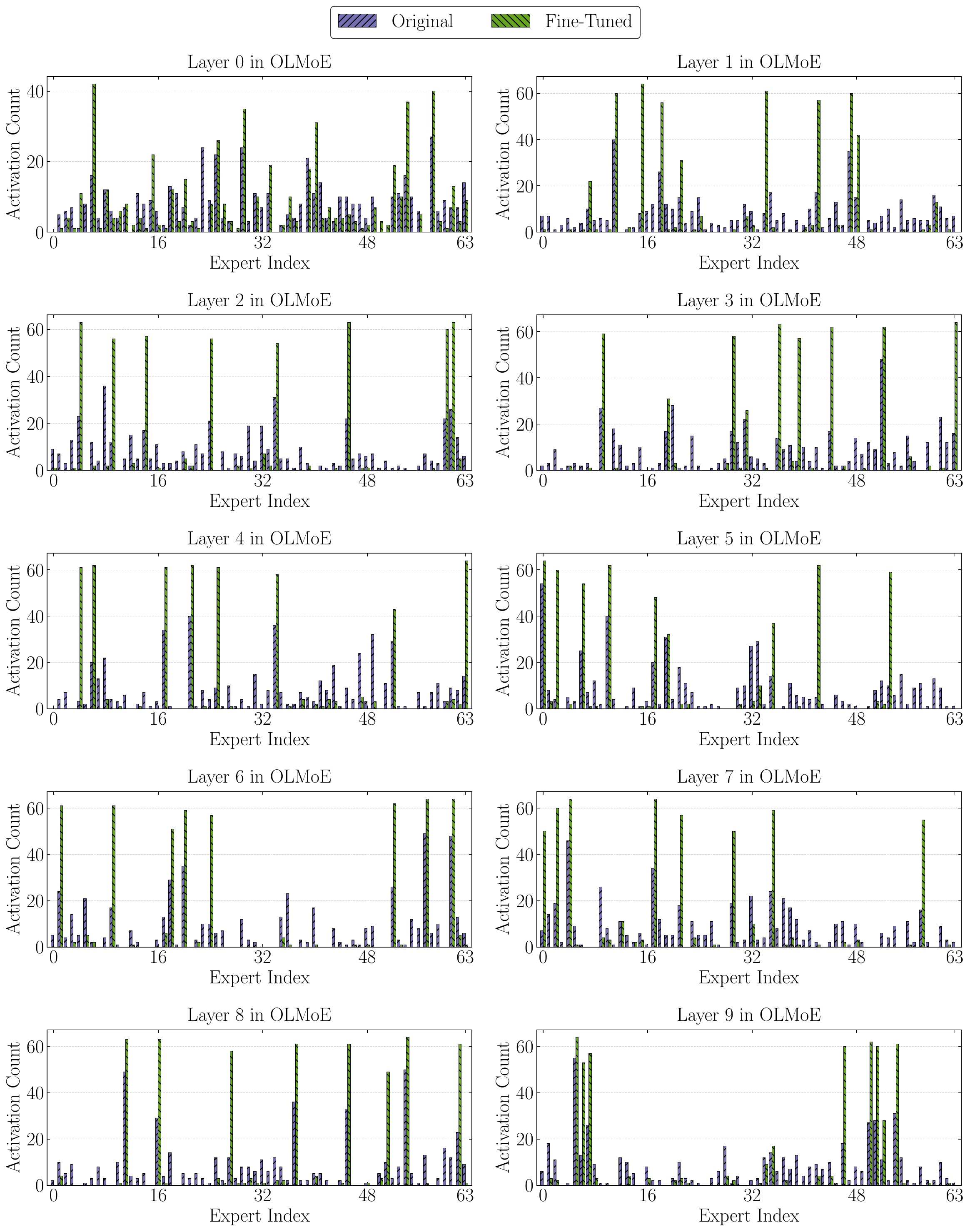}
  \caption{Activations of experts for a single sequence in the first $10$ layers of OLMoE.}
  \label{fig:olmoe_acts}
  \vspace{-5mm}
\end{figure*}
\newpage 
\begin{figure*}[!htbp]
  \centering
  \includegraphics[width=0.95\textwidth]{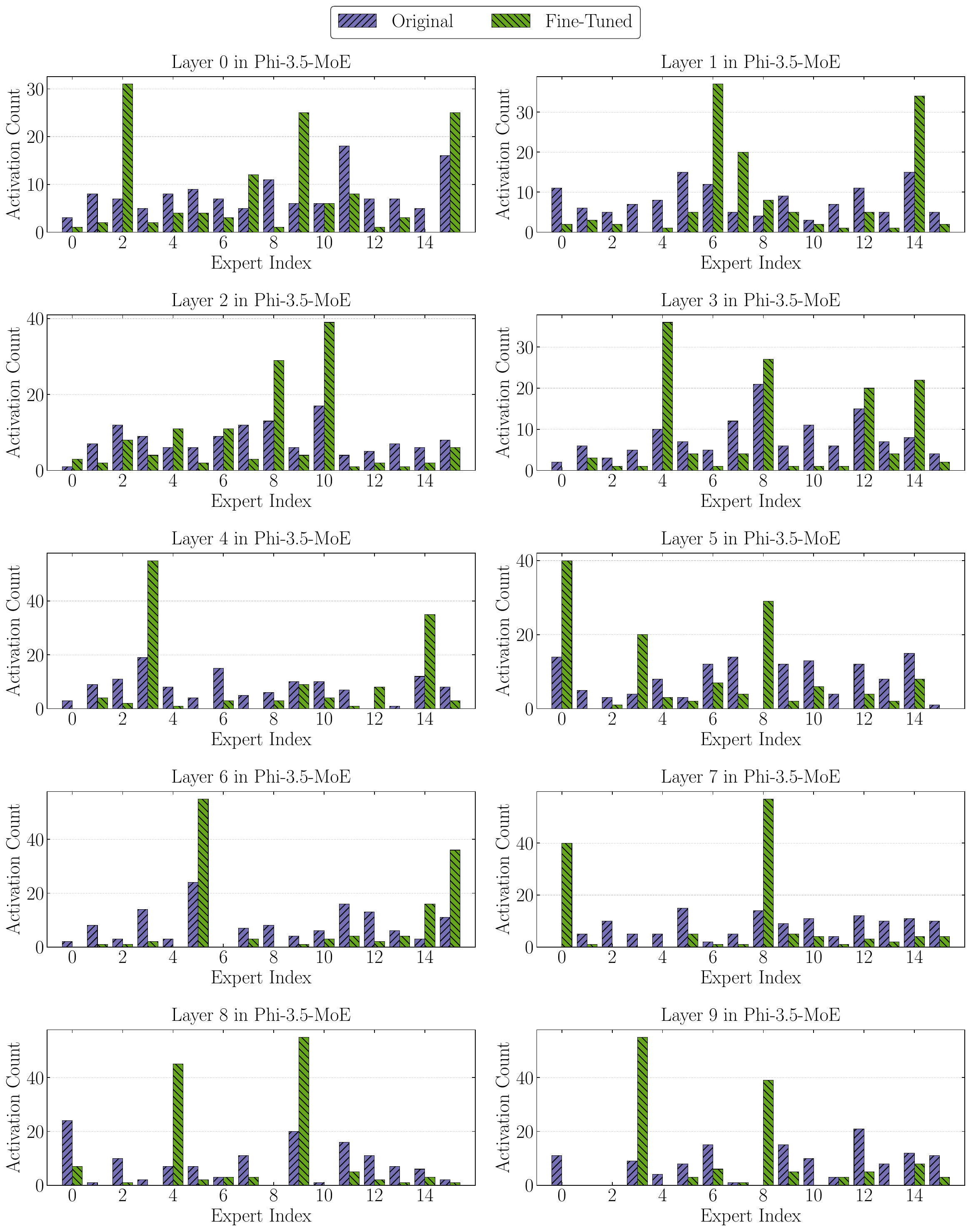}
  \caption{Activations of experts for a single sequence in the first $10$ layers of Phi-3.5-MoE.}
  \label{fig:phi3_acts}
  \vspace{-5mm}
\end{figure*}
\newpage 
\begin{figure*}[!htbp]
  \centering
  \includegraphics[width=0.95\textwidth]{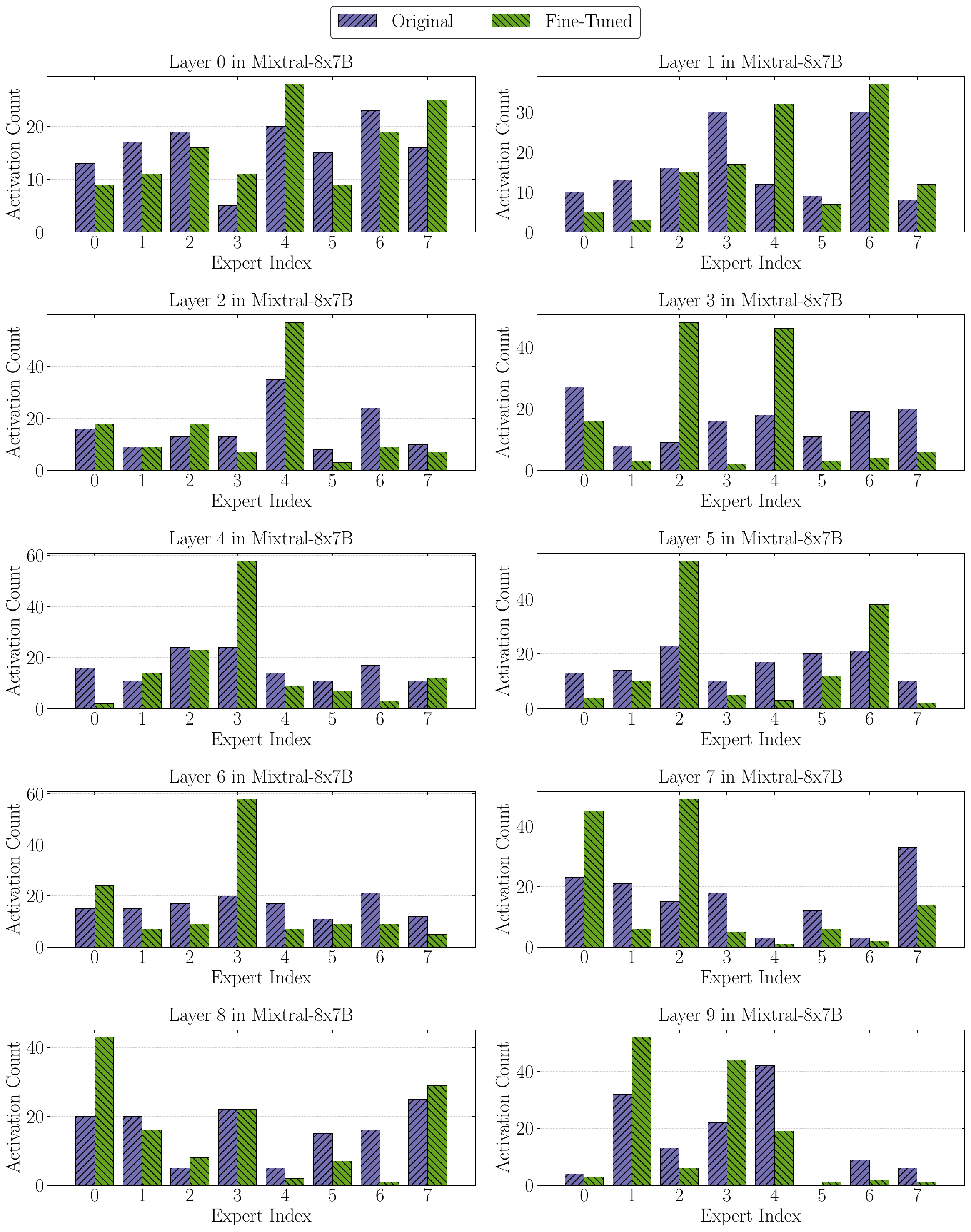}
  \caption{Activations of experts for a single sequence in the first $10$ layers of Mixtral-8x7B.}
  \label{fig:mixtral_acts}
  \vspace{-5mm}
\end{figure*}
\newpage 
\begin{figure*}[!htbp]
  \centering
  \includegraphics[width=0.95\textwidth]{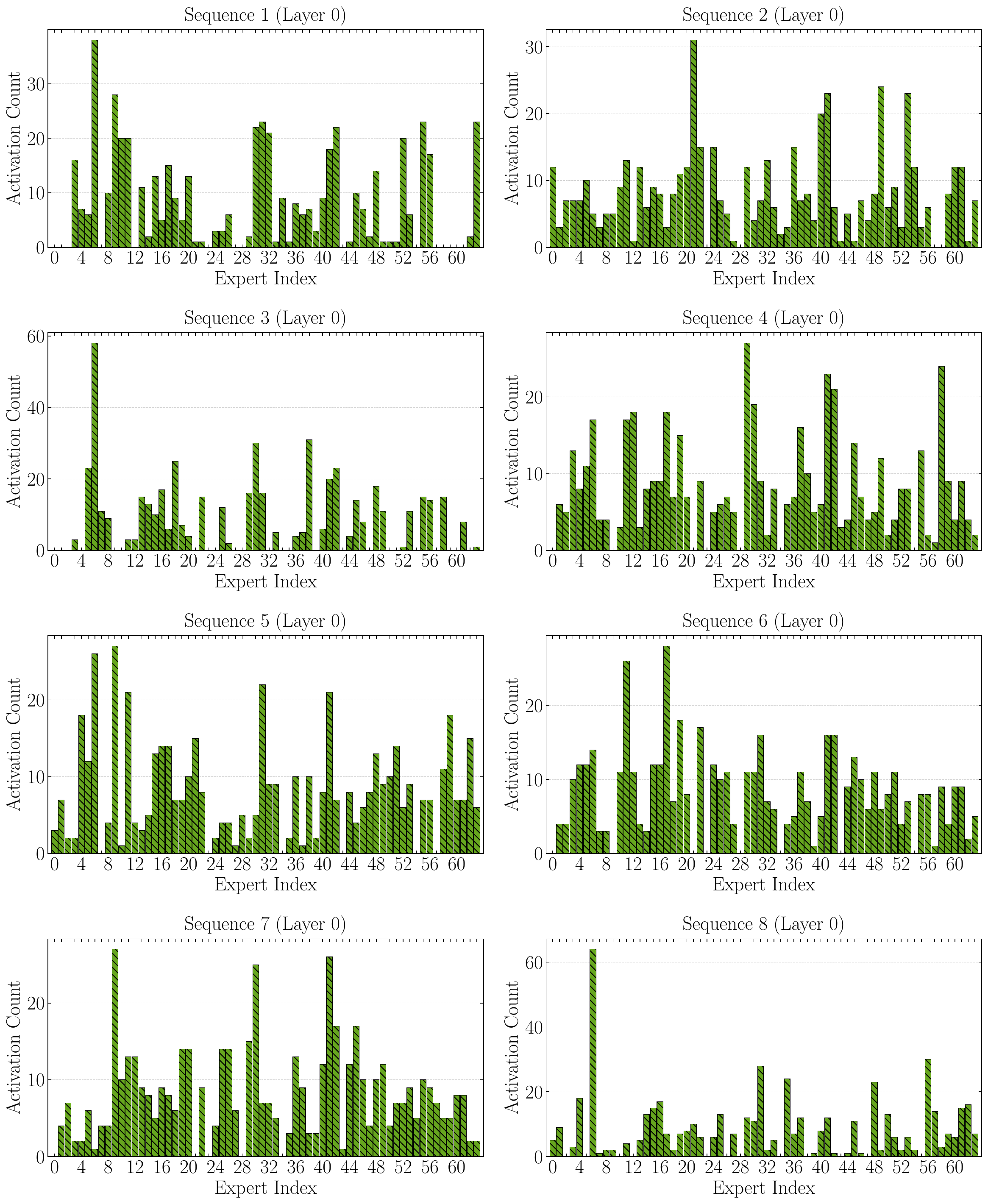}
  \caption{Activations of experts for $8$ different sequences in layer $0$ of OLMoE.}
  \label{fig:olmoe_seq_acts}
\end{figure*}
\noindent Figures \ref{fig:olmoe_acts}, \ref{fig:phi3_acts}, and \ref{fig:mixtral_acts} demonstrate the effects of fine-tuning on expert activations for OLMoE, Phi-3.5-MoE, and Mixtral-8x7B, respectively. Across all model architectures, fine-tuning skews activations towards a few highly preferred experts, justifying fine-tuning as a meaningful approach to make routing more predictable for MoE models. Additionally, Figure \ref{fig:olmoe_seq_acts} demonstrates that routing still remains diverse across multiple sequences, suggesting that the combined effect of the fine-tuning procedure presented in \AlgName\ is to create sequence-specific skew but retain global expert usage diversity. 

\subsection{Ablation on GPU VRAM Budget}
\label{subsec:cache_size}
\begin{figure}[htbp]
  \centering

  \begin{subfigure}[t]{0.92\textwidth}
    \centering
    \includegraphics[width=\textwidth]{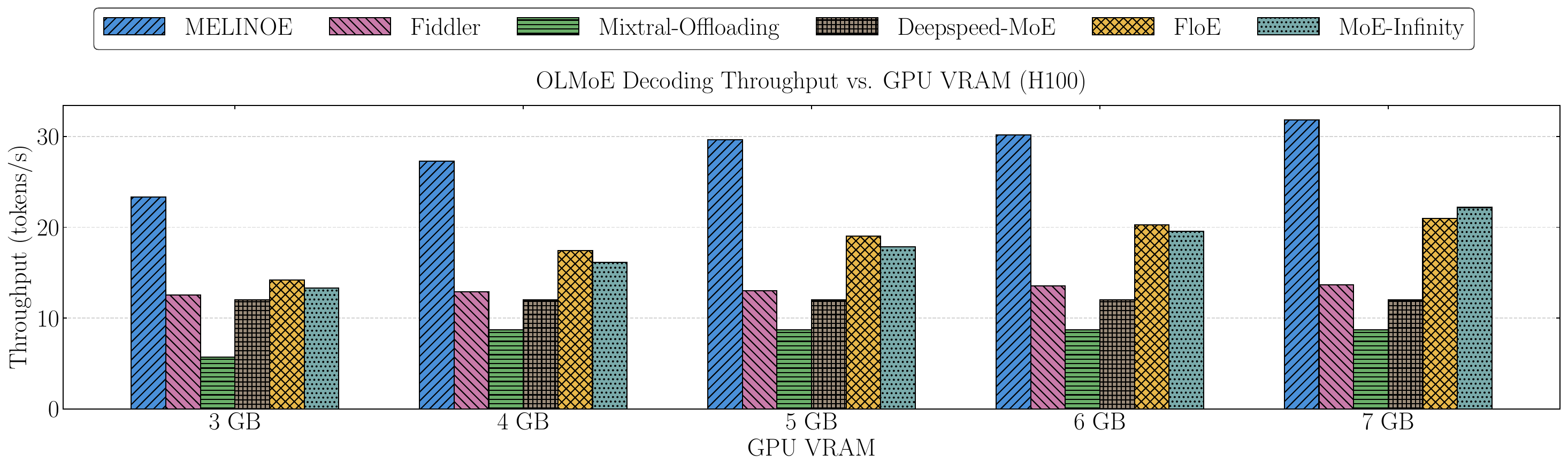}
    \label{fig:olmoe_vram}
  \end{subfigure}

  \vspace{-4mm}

  \begin{subfigure}[t]{0.92\textwidth}
    \centering
    \includegraphics[width=\textwidth]{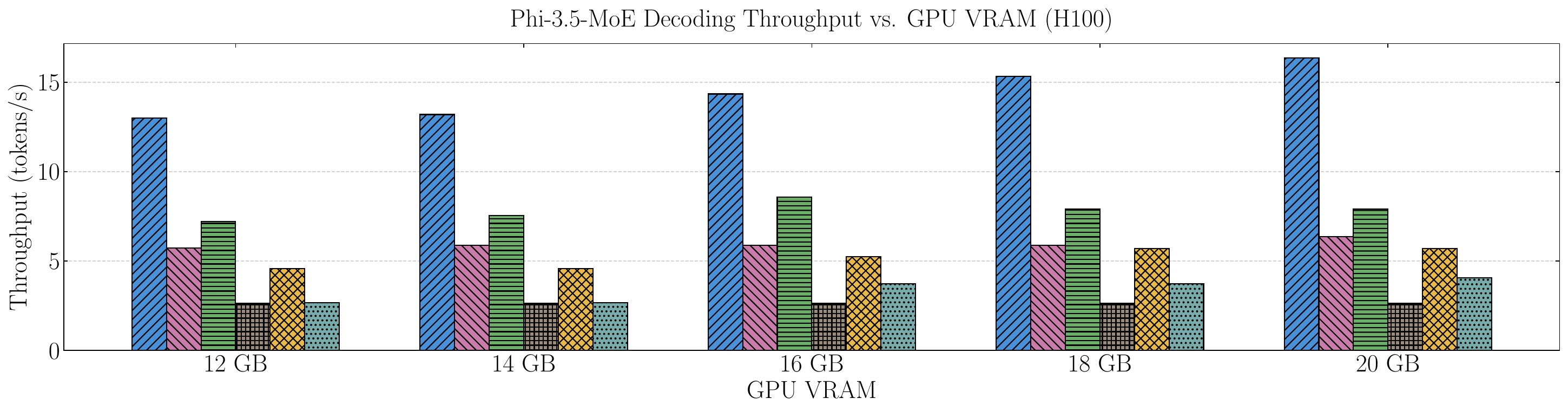}
    \label{fig:phi3_vram}
  \end{subfigure}

  \vspace{-4mm}

  \begin{subfigure}[t]{0.92\textwidth}
    \centering
    \includegraphics[width=\textwidth]{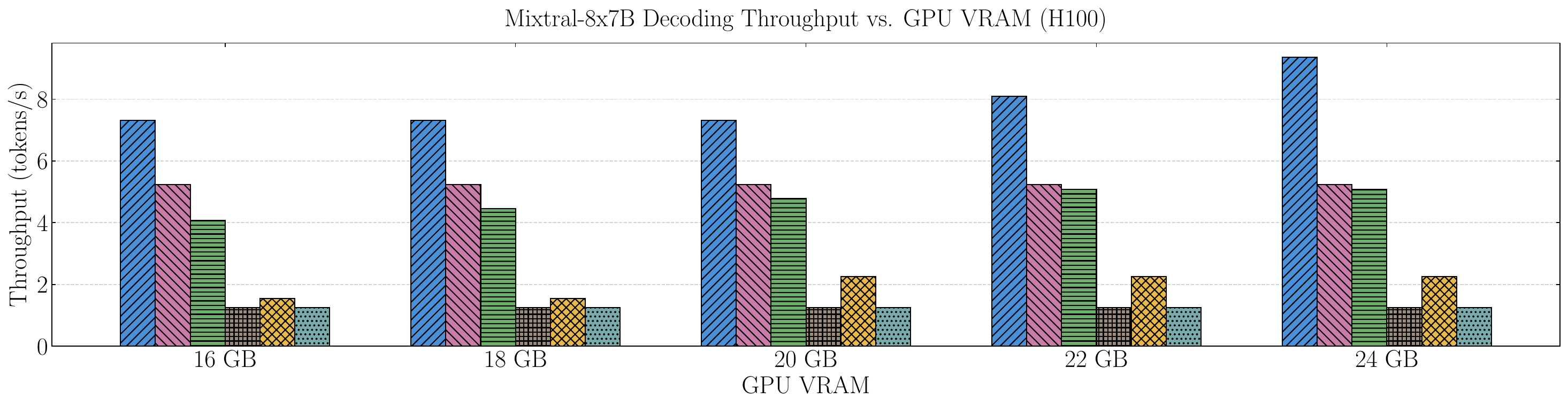}
    \label{fig:mixtral_vram}
  \end{subfigure}
  \vspace{-4mm}
  \caption{Throughput of baselines under GPU VRAM restrictions on H100 across three MoE models.}
  \label{fig:vram_all}
\end{figure}

\noindent Figure \ref{fig:vram_all} displays the impact of GPU VRAM on decoding throughput. Across all VRAM budgets, \AlgName\ outperforms prior offloading baselines, demonstrating its robustness across a  diverse array of downstream deployments. 

\subsection{Ablation on Quantized Experts}
\begin{table}[!htbp]
\centering
\footnotesize
\setlength{\tabcolsep}{5.0pt}
\renewcommand{\arraystretch}{1.12}
\caption{Impact of quantized experts for OLMoE. We report the number of GPU experts per layer and decoding throughput (tokens/s).}
\label{tab:quant_experts_olmoe_multi}

\begin{tabularx}{\linewidth}{@{}l *{4}{>{\centering\arraybackslash}X}@{}}
\specialrule{0.09em}{0.2em}{0.2em}
\rowcolor{gray!12}
& \multicolumn{2}{c}{\textbf{Dolly15K}} & \multicolumn{2}{c}{\textbf{GSM8K}} \\
\specialrule{0.06em}{0.15em}{0.15em}
\rowcolor{gray!6}
&
\makecell{\textbf{GPU-Resident}\\\textbf{Experts / Layer}} &
\textbf{Throughput} &
\makecell{\textbf{GPU-Resident}\\\textbf{Experts / Layer}} &
\textbf{Throughput} \\
\specialrule{0.06em}{0.15em}{0.15em}

\textbf{Base Model} &
$8$ & $13.98$ &
$8$ & $13.66$ \\

\textbf{Base Model + Quantized Experts} &
$24$ & $15.80$ &
$24$ & $17.95$ \\

\textbf{Fine-Tuned Model} &
$8$ & $22.77$ &
$8$ & $19.89$ \\

\textbf{Fine-Tuned Model + Quantized Experts} &
$24$ & $25.99$ &
$24$ & $28.59$ \\

\specialrule{0.09em}{0.2em}{0.2em}
\end{tabularx}
\end{table}
\noindent Quantizing experts to $\texttt{HQQ INT4}$ allows more experts to remain in GPU VRAM. In Table \ref{tab:quant_experts_olmoe_multi}, we analyze the effects of quantization on throughput. Each setting has roughly equal GPU VRAM usage. Quantization improves throughput, but its benefit is not proportional to the increase in resident experts due to compute overhead. The fine-tuned model with $8$ resident experts has greater throughput than the quantized base model with $24$ resident experts.

\subsection{Ablation on Soft Cache Capacity in Loss}
\label{subsec:soft_cache}
\begin{figure*}[htbp]
  \centering
  \includegraphics[width=0.65\textwidth]{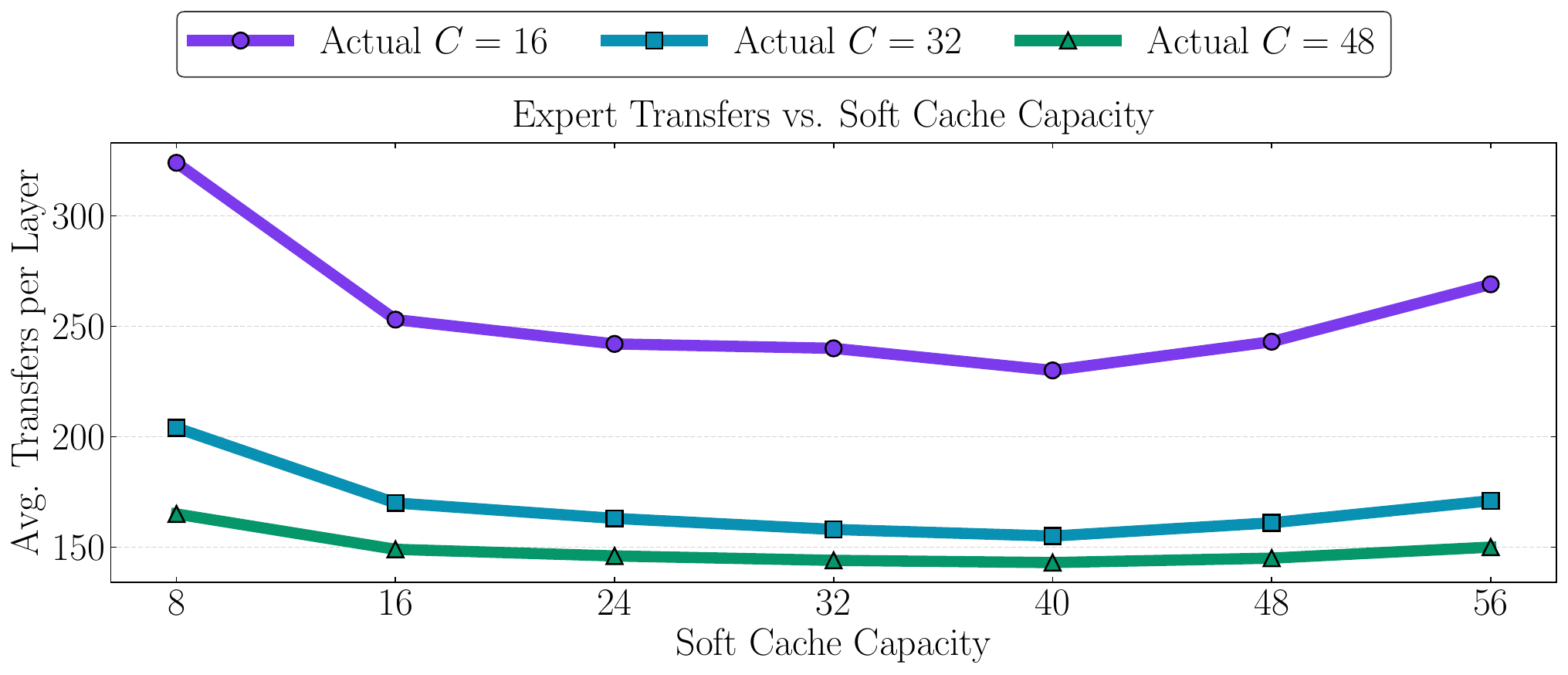}
  \caption{Transfers per layer with fine-tuned model using different soft cache capacities (OLMoE, 64 output tokens).}
  \vspace{-2mm}
  \label{fig:soft_cache}
\end{figure*}
\noindent In Figure \ref{fig:soft_cache}, we evaluate the effect of the soft cache capacity used in fine-tuning on the average number of transfers per layer in the downstream evaluation. We consider three different cache budgets during evaluation, $C=16$, $C=32$, and $C=48$. Performance is noticeably worse when the soft cache capacity is set too low as transfers are dominated by forced evictions, making subtle routing choices more difficult to learn. There is also a slight degradation in the number of transfers when the soft cache budget is set too high as few transfers occur in this regime so transfer penalties remain limited.   

\subsection{Ablation on Loss Function Decay Factor}
\label{subsec:decay}
\begin{figure*}[htbp]
  \centering
  \includegraphics[width=0.65\textwidth]{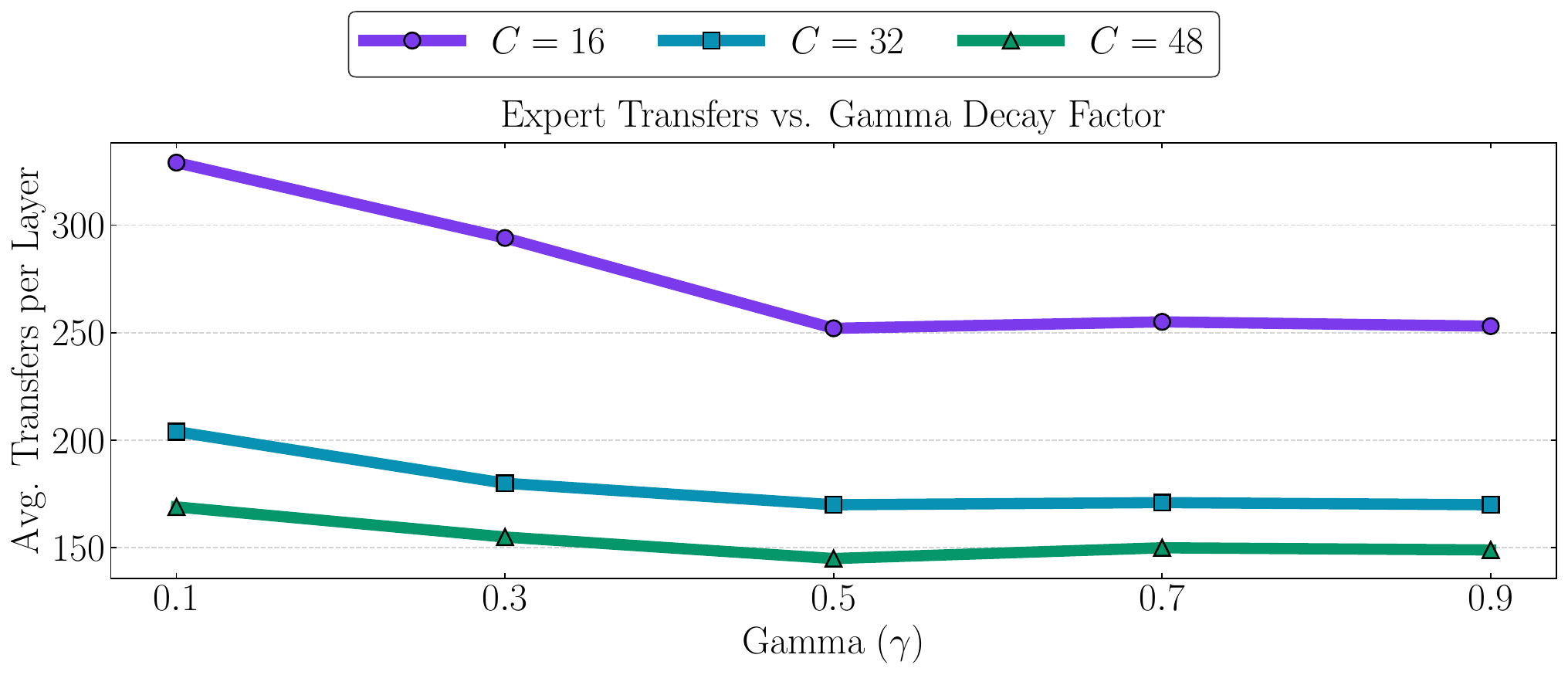}
  \caption{Transfers per layer with fine-tuned model using different $\gamma$ decay factors (OLMoE, 64 output tokens).}
  \vspace{-2mm}
  \label{fig:gamma_cache}
\end{figure*}
\noindent Figure \ref{fig:gamma_cache} studies the decay factor $\gamma$ used in the cache simulation loss, which controls how long past expert usage is factored into current eviction decisions. We find that transfers are high when $\gamma$ is too small but decrease rapidly as $\gamma$ increases across all cache budgets. Thus, overly aggressive decay makes routing decisions too myopic when using an LFU eviction policy. 

\subsection{Ablation on Cache Eviction Policy}
\begin{table}[!htbp]
\centering
\footnotesize
\setlength{\tabcolsep}{7.0pt}
\renewcommand{\arraystretch}{1.12}
\caption{Transfers per layer for OLMoE under different decay factors $\gamma$ and different eviction policies during inference.}
\label{tab:gamma_eviction_ablation}

\begin{tabular}{@{}lcc@{}}
\specialrule{0.09em}{0.2em}{0.2em}
\rowcolor{gray!12}
 & \textbf{LRU Eviction Policy} & \textbf{LFU Eviction Policy} \\
\specialrule{0.06em}{0.15em}{0.15em}

\textbf{Fine-Tuned w/ $\gamma=0.1$} & $314$ & $329$ \\
\textbf{Fine-Tuned w/ $\gamma=0.3$} & $287$ & $294$ \\
\textbf{Fine-Tuned w/ $\gamma=0.5$} & $262$ & $252$ \\
\textbf{Fine-Tuned w/ $\gamma=0.7$} & $264$ & $255$ \\
\textbf{Fine-Tuned w/ $\gamma=0.9$} & $262$ & $253$ \\

\specialrule{0.09em}{0.2em}{0.2em}
\end{tabular}
\end{table}
\noindent Finally, in Table \ref{tab:gamma_eviction_ablation}, we quantify the impact of the cache eviction policy on the number of transfers per layer. For smaller values of $\gamma$, an LRU policy results in fewer cache transfers as the model was fine-tuned with a more reactive cache simulation loss. However, using a larger $\gamma$ with an LFU policy results in the fewest number of transfers overall. 

\end{document}